\documentclass[11pt]{article} 
\pdfoutput=1 
\usepackage{comment, fullpage}
\usepackage{amsthm}
\usepackage{amsmath}
\usepackage{amssymb}
\usepackage{graphicx}
\usepackage{url}
\usepackage{setspace}
\usepackage{algorithm,algorithmic}
\usepackage{hyperref}
\usepackage{natbib}
\usepackage{framed}
\usepackage{xcolor}
\usepackage{soul}

\newtheorem{theorem}{Theorem}

\newtheorem{lemma}{Lemma}

\newenvironment{proofof}[1]{\par\noindent{\bfseries\upshape
  Proof #1\ }}{\hfill\qed\\[2mm]}

\newcommand{\E}{\mathbb{E}}

\newcommand{\R}{\Re}

\renewcommand{\P}{\mathcal{P}}

 \newcommand{\eat}[1]{}

\usepackage{wrapfig}
\usepackage{tikz}

\newcommand{\cA}{\mathrm{cA}}
\newcommand{\dA}{\mathrm{dA}}
\newcommand{\cB}{\mathrm{cB}}
\newcommand{\dB}{\mathrm{dB}}
\newcommand{\cD}{\mathrm{cD}}
\newcommand{\dD}{\mathrm{dD}}
\newcommand{\cC}{\mathrm{cC}}
\newcommand{\dC}{\mathrm{dC}}

\newcommand{\A}{\mathcal{A}}
\newcommand{\argmin}{\mathrm{argmin}}
\newcommand{\hL}{\mathcal{L}}

\newcommand{\sign}{\mathrm{sign}}

\newcommand{\error}{\mathrm{err}}
\newcommand{\erf}{\mathrm{erf}}

\newcommand{\cb}{2.3463}
\newcommand{\ao}{0.0121608} % \alpha_0: angle for the initialization
\newcommand{\Lw}{0.665769} % L(w^*) coefficient
\newcommand{\gen}{10^{-8}} % L(w^*) coefficient

\renewcommand{\vec}[0]{}

\newif\ifcomment
\commentfalse

\newcommand{\opt}{\mathrm{OPT}}
\newcommand{\mnc}{\mathrm{\eta}}

\title{Efficient Learning of Linear Separators under Bounded Noise}
\usepackage{times}

\author{Pranjal Awasthi \\ \small{pawashti@cs.princeton.edu} \and
Maria-Florina Balcan  \\ \small{ninamf@cs.cmu.edu} \and
Nika Haghtalab \\ \small{nhaghtal@cs.cmu.edu}\and
Ruth Urner  \\ \small{rurner@tuebingen.mpg.de}
}
\begin{document}

\maketitle

\begin{abstract}
We study the learnability of linear separators in $\Re^d$ in the presence of bounded  (a.k.a Massart) noise.
 This is a  realistic generalization of the random classification noise model, where the adversary can flip each example $x$ with probability $\eta(x) \leq \eta$. We provide the first polynomial time algorithm that can learn linear separators to arbitrarily small excess error in this noise model under the uniform distribution over the unit ball in $\Re^d$, for some constant value of $\eta$. While widely studied in the statistical learning theory community in the context of getting faster convergence rates, computationally efficient algorithms in this model had remained elusive. Our work provides the first evidence that one can indeed design  algorithms achieving arbitrarily small excess error in  polynomial time  under this realistic noise model and thus opens up a new and exciting line of research.

We additionally provide lower bounds showing that popular algorithms such as hinge loss minimization and averaging cannot lead to arbitrarily small excess error under Massart noise, even under the uniform distribution. Our work instead, makes use of a margin based technique developed in the context of active learning. As a result, our algorithm is also an active learning algorithm with label complexity that is only a logarithmic the desired excess error $\epsilon$. 
\end{abstract}

\section{Introduction}

\noindent{\bf Overview}~~
Linear separators are the most popular classifiers studied in both the theory and practice of machine learning.
Designing noise tolerant, polynomial time learning algorithms that achieve arbitrarily small excess error rates for linear separators is a long-standing question in learning theory. In the absence of noise (when the data is realizable) such algorithms exist via linear programming~\citep{Cristianini00}. However, the problem becomes significantly harder in the presence of label noise. In particular, in this work we are concerned with designing algorithms that can achieve error $\opt +\epsilon$ which is arbitrarily close to $\opt$, the error  of the best linear separator, and run in time polynomial in $\frac 1 \epsilon$ and $d$ (as usual, we call $\epsilon$  the {\em excess error}). Such strong guarantees are only known for the well studied random classification noise model~\citep{blum1998polynomial}. In this work, we provide the first algorithm that can achieve arbitrarily small excess error, in truly polynomial time, for bounded noise, also called Massart noise~\citep{massart2006}, a much more  realistic and
widely studied noise model in statistical learning theory~\citep{bbl05}.
We additionally show strong lower bounds under the same noise model for two other computationally efficient learning algorithms (hinge loss minimization and the averaging algorithm), which could be of independent interest.

\smallskip

\noindent{\bf Motivation}~~
The work on computationally efficient algorithms for learning halfspaces has focused on two different extremes. On one hand, for the very stylized random classification noise model (RCN),  where each example $\vec x$ is flipped independently with equal probability $\mnc$, several works have provided computationally efficient algorithms that can achieve arbitrarily small excess error in polynomial time ~\citep{blum1998polynomial,servedio2001efficient,vitalynina13} --- note that all these results crucially exploit the high amount of symmetry present in the RCN noise.
At the other extreme, there has been significant work on much more difficult and adversarial noise models, including the agnostic model~\citep{KearnsSS94}
 and malicious noise models~\citep{Kearns-li:93}.
The best results here however, not only require additional distributional assumptions about the marginal over the instance space,
 but they only achieve much weaker multiplicative approximation guarantees~\citep{kalai2008agnostic, KLS09,awasthi2014power}; for example, the best result of this form for the case of
uniform distribution over the unit sphere $S_{d-1}$
achieves excess error  $c \opt$~\citep{awasthi2014power}, for some large constant $c$.  While interesting from a technical point of view, guarantees of this form are somewhat troubling
 from a statistical point of view, as they are inconsistent, in the sense there is a barrier  $O(\opt)$, after which we cannot prove that the excess error
further decreases as we get more and more samples.
In fact, recent evidence shows that this is unavoidable for polynomial time algorithms for such adversarial noise models~\citep{amit2014stoc}.

\smallskip

\noindent{\bf Our Results}~~
In this work we identify  a realistic and widely studied noise model in the statistical learning theory, the so called Massart noise~\citep{bbl05},
 for which we can prove much stronger guarantees.
Massart noise can be thought of as a generalization of the random classification noise model where the label of each example $x$ is flipped independently
with probability $\mnc(x) < 1/2$.
The adversary has control over choosing a different noise rate $\eta(x) \leq \mnc$ for every example $x$ with the only constraint that $\eta(x) \leq \mnc$.
From a statistical point of view, it is well known that under this model, we can get faster rates compared to worst case joint distributions~\citep{bbl05}.
In computational learning theory, this noise model was also studied, but under the name of malicious misclassification
noise \citep{rivest1994formal,sloan1996pac}. However due to its highly unsymmetric nature, til date, computationally efficient learning algorithms in this model have remained elusive.
In this work, we provide the first computationally efficient algorithm achieving arbitrarily small excess error for learning linear separators.

Formally, we show that there exists a polynomial time algorithm that can  learn linear separators to  error $\opt + \epsilon$ and run in $\textrm{poly}(d, \frac{1}{\epsilon})$
when the underlying distribution is the uniform distribution over the unit ball in $\Re^d$ and the noise of each example is upper bounded by a constant $\mnc$ (independent of the dimension).

As mentioned earlier, a result of this form was only known for random classification noise.
 From a technical point of view,
as opposed to random classification noise, where the error of each classifier scales uniformly under the observed labels,
 the observed error of classifiers under Masasart noise could change drastically in a non-monotonic fashion. This is due to the fact that the adversary has control over choosing a different noise rate $\eta(x) \leq \mnc$ for every example $x$. As a result, as we show in our work (see Section~\ref{sec:average}), standard algorithms such as the averaging algorithm~\citep{servedio2001efficient} which work for random noise can only achieve a much poorer excess error~(as a function of $\mnc$) under Massart noise. Technically speaking, this is due to the fact that Massart noise can introduce high correlations between the observed labels and the component orthogonal to the direction of the best classifier.

In face of these challenges, we take an entirely different approach than previously considered for random classification noise.
Specifically, we analyze a recent margin based algorithm of~\cite{awasthi2014power}.
This algorithm was designed for learning linear separators under agnostic and malicious noise models,
and it was shown to achieve an excess error of  $c\opt$ for a constant $c$.
 By using  new structural insights, we show that  there exists a {\em constant} $\mnc$ (independent of the dimension), so that if we use
 Massart noise where the flipping probability is upper bounded by $\mnc$, we can use a modification of the algorithm in~\cite{awasthi2014power} and achieve arbitrarily small excess error.
   One way to think about this result is that we define an adaptively chosen sequence of hinge loss minimization problems around smaller and smaller bands
 around the current guess for the target. We show by relating the hinge loss and 0/1-loss together with a careful localization analysis that
 these will direct us closer and closer to the optimal classifier, allowing us to achieve arbitrarily small excess error rates in polynomial time.

Given that our algorithm is an adaptively chosen sequence of hinge loss minimization problems, one might wonder what guarantee one-shot hinge loss minimization could provide.
In Section~\ref{hinge},
we show a strong negative result: for every $\tau$, and $\mnc \leq 1/2$, there is a noisy distribution $\tilde D$ over $\Re^d \times \{0,1\}$ satisfying Massart noise with parameter $\mnc$ and an $\epsilon>0$, such that $\tau$-hinge loss minimization returns a classifier with excess error $\Omega(\epsilon)$.
This  result could be of independent interest. While there exists  earlier work showing that hinge loss minimization can lead to classifiers of large $0/1$-loss \citep{Ben-DavidLSS12},
the lower bounds in that paper  employ distributions with \emph{significant mass} on \emph{discrete points} with \emph{flipped label} (which is not possible under Massart noise) at a very \emph{large distance} from the optimal classifier. Thus, that result makes strong use of the hinge loss's sensitivity to errors at large distance.
Here, we show that hinge loss minimization is bound to fail under much more benign conditions.

One appealing  feature of our result is the algorithm we analyze is in fact naturally adaptable to the  active learning or selective sampling scenario
(intensively studied in recent years~\citep{Hanneke07,sanjoy-coarse,hanneke-survey}, where
the learning algorithms only receive
the classifications of examples when they ask for them. We show
that, in this model, our algorithms achieve a label complexity whose
dependence on the error parameter $\epsilon$ is polylogarithmic
(and thus exponentially better than that of any passive algorithm).
This provides the first polynomial-time
active learning algorithm for learning linear
separators under Massart noise.
We note that prior to our work only inefficient algorithms could achieve the desired label complexity
 under Massart noise~\citep{balcan2007margin,hanneke-survey}.

\smallskip

\noindent{\bf Related Work}
The agnostic noise model is notoriously hard to deal with computationally and there is significant evidence that achieving arbitrarily small excess
 error in polynomial time is hard in this model~\citep{ABSS97, GR06, amit2014stoc}.
For this model, under our distributional assumptions, ~\citep{kalai2008agnostic} provides an algorithm that learns linear separators in $\Re^d$ to excess error at most $\epsilon$,  but whose running time $poly(d^{\exp(1/\epsilon)})$.
Recent work show evidence that the exponential dependence on
  $1/\epsilon$ is unavoidable in this case~\citep{klivans2014embedding} for the agnostic case. We side-step this by considering a more structured, yet realistic noise model.

Motivated by the fact that many modern machine learning applications
have massive amounts of unannotated or unlabeled data, there has been significant interest in designing active learning algorithms that most efficiently utilize the available data,
 while minimizing the need for human intervention.
Over the past decade there has been substantial  progress on understanding the underlying statistical principles of active learning, and several general characterizations have been developed for
describing when active learning could have an advantage over the classical passive supervised learning paradigm both in the noise free settings and in the agnostic case
~\citep{QBC,sanjoy-coarse,Balcan06,balcan2007margin,Hanneke07,dhsm,CN07,sanjoy11-encyc,hanneke-survey}.
However, despite many efforts, except for very simple noise models (random classification noise~\citep{vitalynina13} and linear noise~\citep{dgs12}),
to date there are no known computationally efficient algorithms with provable guarantees in the presence of Massart noise that can achieve arbitrarily small excess error.

We note that work of~\cite{HY14}  provides computationally efficient algorithms for both passive and active learning under the
assumption that the hinge loss (or other surrogate loss) minimizer aligns with the minimizer of
the 0/1-loss. In our work (Section~\ref{hinge}),  we show that this is not the case under Massart noise even when the marginal over the instance space is uniform, but still provide a computationally efficient algorithm for this much more challenging setting.

\section{Preliminaries}
We consider the binary classification problem; that is, we work on the problem of predicting a binary label $y$  for a given instance $\vec x$.
We assume that the data points $(\vec x, y)$ are drawn from an unknown underlying distribution $\tilde D$ over $X\times Y$, where  $X= \R^d$ is the instance space and  $Y =\{-1, 1\}$ is the label space. For the purpose of this work, we consider distributions where  the marginal of $\tilde D$ over $X$ is a uniform distribution on a $d$-dimensional unit ball.
We work with the class of all homogeneous halfspaces, denoted by $\mathcal{H}= \{ \sign(\vec w \cdot \vec x):~ \vec w\in \R^d\}$. For a given halfspace $\vec w\in \mathcal{H}$, we define the error of $\vec w$ with respect to $\tilde D$, by $\error_{\tilde D}(\vec w) = \Pr_{(\vec x, y)\sim \tilde D}[\sign(\vec w\cdot \vec x)\neq y]$.

We examine learning halfspaces in the presence of Massart noise. In this setting, we assume that the Bayes optimal classifier is a linear separator $\vec w^*$. Note that $\vec w^*$ can have a non-zero error. Then Massart noise with parameter $\beta>0$ is a condition such that for all $\vec x$, the conditional label probability is such that
\begin{equation}\label{eq:active:massart}
 |\Pr(y=1|\vec x) - \Pr(y=-1|\vec x)| \geq \beta.
\end{equation}
Equivalently, we say that $\tilde D$ satisfies Massart noise with parameter $\beta$, if an adversary construct $\tilde D$ by first taking  the distribution $D$ over  instances $(\vec x, \sign(\vec w^*\cdot \vec x))$ and then flipping the label of an instance $\vec x$ with probability \emph{at most} $\frac{1-\beta}{2}$. \footnote{Note that the relationship between Massart noise parameter $\beta$, and the maximum flipping probability discussed in the introduction $\mnc$, is  $\mnc = \frac{1-\beta}{2}$.}
Also note that under distribution $\tilde D$,  $\vec w^*$ remains the Bayes optimal classier. 
In the remainder of this work, we refer to $\tilde D$ as the  ``noisy'' distribution and to distribution $D$ over instances $(\vec x, \sign(\vec w^*\cdot \vec x))$ as the ``clean'' distribution.

Our goal is then to find a halfspace $\vec w$ that has small excess error, as compared to the Bayes optimal classifier $\vec w^*$. That is, for any $\epsilon >0$, find a halfspace $\vec w$, such that $\error_{\tilde D}(\vec w) - \error_{\tilde D}(\vec w^*) \leq \epsilon$.  
Note that the excess error of any classifier $\vec w$ only depends on the points in the region where $\vec w$ and $\vec w^*$ disagree. So,  $\error_{\tilde D}(\vec w) - \error_{\tilde D}(\vec w^*)\leq \frac{\theta(\vec w,\vec w^*)}{\pi}$. 
Additionally, under Massart noise the amount of noise in the disagreement region is also bounded by $\frac{1-\beta}{2}$. It is not difficult to see that under Massart noise, 
\begin{equation}
 \beta ~\frac{\theta(\vec w, \vec w^*)}{\pi} \leq  \error_{\tilde D}(\vec w) - \error_{\tilde D}(\vec w^*). \label{eq:angle-excess}
\end{equation}

In our analysis, we frequently examine the region  within a certain margin of a halfspace. For a halfspace $\vec w$ and margin $b$, let $S_{\vec w, b}$ be the set of all points that fall within a margin $b$ from $\vec w$, i.e., $S_{\vec w, b} = \{\vec x: ~ |\vec w \cdot \vec x| \leq b\}$.
For distributions $\tilde D$ and $D$, we indicate the distribution conditioned on $S_{\vec w, b}$ by $\tilde D_{\vec w, b}$ and $D_{\vec w, b}$, respectively.
In the remainder of this work, we refer to the region $S_{\vec w, b}$ as ``the band''.

In our analysis, we use hinge loss, as a convex surrogate function for the 0/1-loss. For a halfspace $\vec w$, we use  $\tau$-normalized  hinge loss that is defined as $\ell(\vec w, \vec x, y) = \max \{0, 1- \frac{(\vec w\cdot \vec x)y}{\tau}\}$.
For a labeled sample set $W$, let $\ell(\vec w, W) = \frac{1}{|W|} \sum_{(\vec x,y)\in W} \ell(\vec w, \vec x, y)$ be the empirical hinge loss of a vector $w$ with respect to $W$. 

\section{Computationally Efficient Algorithm for Massart Noise} \label{sec:margin}
In this section, prove our main result for learning half-spaces in presence of Massart noise.
We focus on the case where $D$ is the uniform distribution on the  $d$-dimensional unit ball.
Our main Theorem is as follows.

\begin{theorem}\label{thm:active:hinge-band} \label{thm:main-intro}
Let the optimal bayes classifier be a half-space denoted by $\vec w^*$. Assume that the massart noise condition holds for some $\beta> 1 - 3.6\times 10^{-6}$.
Then for any $\epsilon$, $\delta>0$, Algorithm~\ref{alg:blocks} with $\lambda = \gen$,  $\alpha_k = 0.038709\pi (1-\lambda)^{k-1}$, $b_{k-1} = \frac{\cb \alpha_k}{\sqrt d}$, and  $\tau_k = \sqrt{2.50306} ~ (3.6 \times 10^{-6})^{1/4} b_{k-1}$,
runs in polynomial time,  proceeds in
$s = O(\log \frac{1}{\epsilon})$ rounds, where in round $k$ it takes $n_k = \mathrm{poly}(d, \exp(k), \log(\frac 1\delta))$ unlabeled samples and $m_k = O(d(d+ \log(k/\delta)))$ labels
and with probability $(1-\delta)$ returns a linear separator  that has excess error (compared to $\vec w^*$) of at most $\epsilon$.
\end{theorem}

Note that in the above theorem and Algorithm~\ref{alg:blocks}, the value of $\beta$ is unknown to the algorithm, and therefore, our results are adaptive to values of $\beta$ within the acceptable range defined by the theorem.

The  algorithm described above is similar to that of \cite{awasthi2014power} and uses an iterative margin-based approach. The algorithm runs for $s = \log_{\frac{1}{1-\lambda}}(\frac 1 \epsilon)$ rounds for a constant $\lambda \in (0, 1]$.
By induction assume that our algorithm
produces a  hypothesis  $\vec w_{k-1}$ at round $k-1$ such that $\theta(\vec w_{k-1}, \vec w^*) \leq \alpha_{k}$. We satisfy the base case by using an algorithm of \cite{KLS09}.
At round $k$, we sample $m_{k}$ labeled examples from the conditional distribution $\tilde D_{\vec w_{k-1}, b_{k-1}}$ which is the uniform distribution over $\{x: |\vec w_{k-1} \cdot x| \leq b_{k-1}\}$.
We then choose $\vec w_{k}$ from the set of all hypothesis $B(\vec w_{k-1}, \alpha_{k}) = \{\vec w:~ \theta(\vec w, \vec w_{k-1}) \leq \alpha_{k} \}$  such that $\vec w_{k}$ minimizes the empirical hinge loss over these examples.  Subsequently, as we prove in detail later, $\theta(\vec w_{k}, \vec w^*) \leq \alpha_{k+1}$. Note that for any $\vec w$, the excess error of $\vec w$ is at most the error of $\vec w$ on $\tilde D$ when the labels  are corrected according to  $\vec w^*$, i.e., $\error_{\tilde D}(\vec w) - \error_{\tilde D}(\vec w^*) \leq \error_D(\vec w)$. Moreover, when $D$ is  uniform,  $\error_D(\vec w) = \frac{\theta(\vec w^*, \vec w)}{\pi}$. Hence,  $\theta(\vec w_{s}, \vec w^*) \leq \pi \epsilon$ implies that $\vec w_s$ has  excess error of at most $\epsilon$.

The algorithm described below was originally introduced to achieve an error of $c \cdot \error(\vec w^*)$ for some constant $c$ in presence of adversarial noise. Achieving a small excess error $\error(\vec w^*) + \epsilon$ is a much more ambitious goal -- one that requires new technical insights. Our two crucial  technical innovations are as follow: We first make a key observation that under Massart noise, the noise rate over any conditional distribution $\tilde{D}$ is still at most $\frac{1-\beta}{2}$. Therefore, as we focus on the distribution within the band, our noise rate does not increase.
Our second technical contribution is a careful choice of parameters.
Indeed the choice of parameters, upto a constant, plays an important role in tolerating a constant amount of Massart noise.
Using these insights, we show that the algorithm by \cite{awasthi2014power} can indeed achieve a much stronger guarantee, namely arbitrarily small excess error in presence of Massart noise.
That is, for any $\epsilon$, this algorithm can achieve error of $\error(\vec w^*)+ \epsilon$ in the presence of Massart noise.

\begin{algorithm}[h]
\caption{\textsc{Efficient Algorithm for Arbitrarily Small Excess Error for Massart Noise}}
\label{alg:blocks}
\textbf{Input:} A distribution $\tilde D$.
An oracle that returns $\vec x$ and an oracle that returns $y$ for a $(\vec x, y)$ sampled from $\tilde D$. Permitted excess error $\epsilon$ and probability of failure $\delta$.\\
\textbf{Parameters:} A learning rate $\lambda$; a sequence of sample sizes $m_k$; a sequence of angles of the hypothesis space $\alpha_k$; a sequence of widths of the labeled space $b_{k}$; a sequence of thresholds of hinge-loss $\tau_k$.
\\
\textbf{Algorithm:}\begin{enumerate}
\item Take $\mathrm{poly}(d, \frac{1}{\delta})$ samples and run $\mathrm{poly}(d, \frac{1}{\delta})$-time algorithm by \cite{KLS09} to find a  half-space $\vec w_0$ with excess error $0.0387089$ such that $\theta(\vec w^*, \vec w_0) \leq 0.038709 \pi$ (Refer to Appendix~\ref{sec:KLS})
\item Draw $m_1$ examples $(\vec x, y)$ from $\tilde D$ and put them into a working set $W$.
\item For $k = 1, \dots, \log_{(\frac{1}{1-\lambda})}(\frac 1 \epsilon)=s$.
  \begin{enumerate}
  \item Find $\vec v _k$ such that $\Vert \vec v_k - \vec w_{k-1} \Vert< \alpha_k$ (as a result $\vec v_k \in B(\vec w_{k-1}, \alpha_k)$), that minimizes the empirical hinge loss over $W$ using threshold $\tau_k$. That is $ \ell_{\tau_k}(\vec v_k, W) \leq \min_{\vec w\in B(w_{k-1}, \alpha_k)}  \ell_{\tau_k}(\vec w, W) + \gen$.
\item Clear the working set $W$.
\item  Normalize $\vec v_k$ to $\vec w_k = \frac{\vec v_k}{\Vert \vec v_k \Vert_2}$. Until $m_{k+1}$ additional examples are put in $W$, draw an example $\vec x$ from $\tilde D$. If $|\vec w_k \cdot  \vec x| \geq  b_k$, then reject $\vec x$, else put $(\vec x, y)$ into $W$.
  \end{enumerate}
\end{enumerate}
\textbf{Output:} Return $\vec w_s$, which has excess error $\epsilon$ with probability $1-\delta$.
\end{algorithm}

\noindent \textbf{Overview of our analysis:}
Similar to \cite{awasthi2014power}, we divide  $\error_D(\vec w_k)$ to two categories; error in the band, i.e., on $\vec x \in S_{w_{k-1}, b_{k-1}}$, and  error outside the band, on $\vec x\not\in S_{w_{k-1}, b_{k-1}}$.
We choose $b_{k-1}$ and $\alpha_k$ such that, for every hypothesis $\vec w \in  B(\vec w_{k-1}, \alpha_{k})$ that is considered at step $k$, the probability mass outside the band such that  $\vec w$ and $\vec w^*$ also disagree is very small (Lemma~\ref{lem:disagree-outside-band}). Therefore, the error associated with the region outside the band is also very small. This motivates the design of the algorithm to only minimize the error in the band.
Furthermore, the probability mass of the band is also small enough such that for $\error_D(\vec w_k)\leq \alpha_{k+1}$ to hold, it suffices for $\vec w_k$ to have a small constant error over the clean distribution restricted to the band, namely  $D_{\vec w_{k-1}, b_{k-1}}$.

This is where minimizing hinge loss in the band comes in. As minimizing the 0/1-loss is NP-hard, an alternative method for  finding $\vec w_k$ with small error in the band is needed. Hinge loss that  is a convex loss function can be efficiently minimized. So, we can efficiently find $\vec w_k$ that minimizes the empirical hinge loss of the sample drawn from $\tilde D_{\vec w_{k-1}, b_{k-1}}$. To allow the hinge loss to remain a faithful proxy of 0/1-loss as we focus on bands with smaller widths, we use a normalized hinge loss function  defined by $\ell_\tau(\vec w, \vec x, y) = \max\{0, 1- \frac{\vec w \cdot \vec xy}{\tau} \}$.

A crucial part of our analysis involves showing that if $\vec w_k$ minimizes the empirical hinge loss of the sample set drawn from $\tilde D_{\vec w_{k-1}, b_{k-1}}$, it indeed has a small 0/1-error on $D_{\vec w_{k-1}, b_{k-1}}$.
To this end, we first show that  when $\tau_k$ is proportional to $b_k$,
the hinge loss of $\vec w^*$ on $D_{\vec w_{k-1}, b_{k-1}}$, which is an upper bound on the 0/1-error of $w_k$ in the band, is itself small (Lemma~\ref{lem:active:L(w*)}). Next, we notice that under Massart noise, the noise rate in any marginal of the distribution is still at most $\frac{1-\beta}{2}$. Therefore, focusing the distribution in the band does not increase the probability of noise in the band. Moreover, the noise points in the band are close to the decision boundary so intuitively speaking, they can not increase the hinge loss  too much. Using these insights we can show that the hinge loss of $\vec w_k$ on $\tilde D_{\vec w_{k-1}, b_{k-1}}$ is close to its hinge loss on $D_{\vec w_{k-1}, b_{k-1}}$ (Lemma~\ref{lem:active:diff-clean-dirty}).

\subsection*{Proof of Theorem~\ref{thm:active:hinge-band} and related lemmas}
To prove  Theorem~\ref{thm:active:hinge-band}, we first  introduce a series of lemmas concerning the  behavior of  hinge loss  in the band. These lemmas build up towards showing that $\vec w_k$ has  error of at most a fixed small constant in the band.

For ease of exposition, for any $k$, let $D_k$ and $\tilde D_k$ represent
 $D_{\vec w_{k-1}, b_{k-1}}$ and  $\tilde D_{\vec w_{k-1}, b_{k-1}}$, respectively, and  $\ell(\cdot)$ represent $\ell_{\tau_k}(\cdot)$.
Furthermore, let $c = \cb$,  such that $b_{k-1} = \frac{c \alpha_k}{\sqrt{d}}$.

Our first lemma, whose proof appears in Appendix~\ref{app:margin_uniform}, provides an upper bound on the  true hinge error of $\vec w^*$ on the clean distribution in the band.
\begin{lemma}\label{lem:active:L(w*)}
$ \E_{(\vec x, y)\sim D_k } \ell(\vec w^*,\vec x, y)\leq   \Lw\frac{\tau}{b}$.
\end{lemma}

The next Lemma compares the true hinge loss of  any $\vec w\in B(\vec w_{k-1}, \alpha_k)$ on two distributions, $\tilde D_k$ and $D_k$.
It is clear that the difference between the hinge loss on these two distributions is entirely attributed to the noise points and their margin from $\vec w$.
A key insight in the proof of this lemma is that as we concentrate in the band, the probability of seeing a noise point remains under $\frac{1-\beta}{2}$. This is due to the fact that under Massart noise,  each label can be changed with probability at most $\frac{1-\beta}{2}$. Furthermore, by concentrating in the band all points are close to the decision boundary of $\vec w_{k-1}$. Since $\vec w$ is also close in angle to $\vec w_{k-1}$, then points in the band are  also close to the decision boundary of $\vec w$.  Therefore the hinge loss of noise points in the band can not increase the total hinge loss of $\vec w$ by too much.
\begin{lemma} \label{lem:active:diff-clean-dirty}
For any $\vec w$ such that $\vec w \in B(\vec w_{k-1}, \alpha_k)$, we have
\[ | \E_{(\vec x,y)\sim D_k} \ell(\vec w, \vec x, y) - \E_{(\vec x,y)\sim \tilde{D}_k} \ell(\vec w, \vec x, y) | \leq 1.092  \sqrt 2 \sqrt{1-\beta} \frac{b_{k-1}}{\tau_k}.
\]
\end{lemma}
\begin{proof}
Let $N$ be the set of noise points. We have,
\begin{align*}
| \E_{(\vec x,y)\sim \tilde D_k} &\ell(\vec w, \vec x, y) - \E_{(\vec x,y)\sim D_k} \ell(\vec w, \vec x, y) |
  = |\E_{(\vec x,y)\in \tilde D_k} \left( \ell(\vec w, \vec x, y) - \ell(\vec w, \vec x, \sign(\vec w^*\cdot \vec x) \right)|\\
  &\leq \E_{(\vec x,y)\sim \tilde D_k} \left( \mathbf{1}_{\vec x\in N} ( \ell(\vec w, \vec x, y) - \ell(\vec w, \vec x, -y)) \right)\\
  &\leq 2  \E_{(\vec x,y)\sim \tilde D_k} \left( \mathbf{1}_{\vec x\in N} \frac{|\vec w\cdot \vec x|}{\tau_k} \right)\\
  &\leq \frac{2}{\tau_k} \sqrt{\Pr_{(\vec x,y)\sim \tilde D_k} ( \vec x\in N)} \times \sqrt{ \E_{(\vec x,y)\sim \tilde D_k} (\vec w\cdot \vec x)^2 }   \quad\text{(By Cauchy Shwarz)}\\
  &\leq \frac{2}{\tau_k}\sqrt{\frac{1-\beta}{2}} \sqrt{ \frac{\alpha_k^2}{d-1} + b_{k-1}^2 }     \quad \text{(By Definition~4.1 of \citep{awasthi2014power} for uniform)}\\
  &\leq \sqrt 2 \sqrt{1-\beta} \frac{b_{k-1}}{\tau_k}  \sqrt{ \frac{d}{(d-1)c^2} + 1 }  \\
  &\leq 1.092  \sqrt 2 \sqrt{1-\beta} \frac{b_{k-1}}{\tau_k}  \qquad \text{(for $d  > 20,  c>1$)}
\end{align*}
\end{proof}

For a labeled sample set $W$ drawn at random from $\tilde D_k$, let $\mathrm{cleaned}(W)$ be the set  of samples with the labels corrected by $\vec w^*$, i.e.,  $\mathrm{cleaned}(W) = \{ (\vec x, \sign(\vec w^* \cdot \vec x)):~ \text{for all } (\vec x, y)\in W  \}$.
Then by standard VC-dimension bounds (Proof included in Appendix~\ref{app:margin_uniform})
there is $m_k \in O(d(d+ \log(k/d)))$ such that for any randomly drawn set $W$ of $m_k$ labeled samples from $\tilde D_k$,  with probability $1- \frac{\delta}{2(k + k^2)}$, for any $\vec w\in B(\vec w_{k-1}, \alpha_k)$,
\begin{align}
| \E_{(\vec x,y)\sim \tilde D_k} \ell(\vec w,\vec x,y) - \ell(\vec w, W)   |  &   \leq \gen  \label{eq:generalization-dirty},\\
| \E_{(\vec x,y)\sim D_k} \ell(\vec w,\vec x,y) - \ell(\vec w, \mathrm{cleaned}(W))    |  &   \leq  \gen. \label{eq:generalization-clean}
\end{align}
Our next lemma is a crucial step in our analysis of Algorithm~\ref{alg:blocks}.
This lemma proves that if $\vec w_k \in B(\vec w_{k-1}, \alpha_k)$ minimizes the empirical hinge loss on the sample drawn from the noisy distribution in the band, namely $\tilde D_{\vec w_{k-1}, b_{k-1}}$, then with high probability $\vec w_k$ also has a small 0/1-error with respect to the clean distribution in the band, i.e.,  $D_{\vec w_{k-1}, b_{k-1}}$.

\allowdisplaybreaks
\begin{lemma} \label{lem:active:error-clean-in-band}
There exists $m_k \in O(d(d+ \log(k/d)))$, such that for a randomly drawn labeled sampled set $W$ of size  $m_k$ from $\tilde D_k$,
and for $\vec w_k$ such that $\vec w_k$  has the minimum  empirical hinge loss on $W$ between the set of all hypothesis in $B(\vec w_{k-1}, \alpha_k)$, with probability $1 - \frac{\delta}{2(k+k^2)}$ ,
\[ \error_{D_k} (\vec w_k) \leq  0.757941 \frac{\tau_k}{b_{k-1}}    + 3.303 \sqrt{1-\beta} \frac{b_{k-1}}{\tau_k} + 3.28 \times \gen.
\]
\end{lemma}
\begin{proofof}{Sketch}
First, we note that the true 0/1-error of $\vec w_k$ on any distribution is at most its true hinge loss on that distribution.
Lemma~\ref{lem:active:L(w*)} provides an upper bound on the  true hinge loss on distribution $D_k$.
Therefore, it remains to create a connection between the empirical hinge loss of $\vec w_k$ on the sample drawn from $\tilde D_k$ to its true hinge loss on distribution $D_k$.
This, we achieve by using the generalization bounds of Equations~\ref{eq:generalization-dirty} and \ref{eq:generalization-clean} to connect the empirical and true hinge loss of $\vec w_k$ and $\vec w^*$, and using Lemma~\ref{lem:active:diff-clean-dirty} to connect the hinge of $\vec w_k$ and $\vec w^*$ in the clean and noisy distributions.
\end{proofof}

\begin{proofof}{\textbf{of Theorem~\ref{thm:active:hinge-band}}}
For ease of exposition, let $c = \cb$.
Recall that $\lambda = \gen$, $\alpha_k = 0.038709\pi (1-\lambda)^{k-1}$, $b_{k-1} = \frac{c \alpha_k}{\sqrt d}$, $\tau_k = \sqrt{2.50306} ~ (3.6 \times 10^{-6})^{1/4} b_{k-1}$, and $\beta> 1 - 3.6\times 10^{-6}$.

Note that for any $\vec w$, the excess error of $\vec w$ is at most the error of $\vec w$ on the clean distribution $D$, i.e., $\error_{\tilde D}(\vec w) - \error_{\tilde D}(\vec w^*) \leq \error_D(\vec w)$. Moreover, for uniform distribution $D$, $\error_D(\vec w) = \frac{\theta(\vec w^*, \vec w)}{\pi}$.  Hence, to show that $\vec w$ has $\epsilon$ excess error, it suffices to show that $\error_D(\vec w) \leq \epsilon$.

Our goal is to achieve excess error of $0.038709 (1-\lambda)^k$ at round $k$. This we do indirectly by bounding  $\error_D(\vec w_k)$ at every step. We use induction.
For $k=0$, we use the algorithm for adversarial noise model by \cite{KLS09}, which can achieve excess error of $\epsilon$ if $ \error_{\tilde D}(\vec w^*) < \frac{\epsilon^2}{256 \log(1/\epsilon)}$ (Refer to Appendix~\ref{sec:KLS} for more details). For Massart noise, $\error_{\tilde D}(\vec w^*) \leq \frac{1-\beta}{2}$. So, for our choice of $\beta$, this algorithm can achieve excess error of $0.0387089$ in $\mathrm{poly}(d, \frac{1}{\delta})$ samples and run-time. Furthermore, using Equation~\ref{eq:angle-excess},  $\theta(\vec w_0, \vec w^*) < 0.038709 \pi$.

Assume that at round $k-1$, $\error_D(\vec w_{k-1}) \leq 0.038709 (1- \lambda)^{k-1}$. We will show that $\vec w_k$, which  is chosen by the algorithm at round $k$, also has $\error_D(\vec w_k) \leq 0.038709 (1- \lambda)^{k}$.

First note that $\error_D(\vec w_{k-1}) \leq 0.038709(1- \lambda)^{k-1}$
implies  $\theta(\vec w_{k-1}, \vec w^*) \leq \alpha_k$. Let $S=S_{\vec w_{k-1}, b_{k-1}}$ indicate the band at round $k$.
We divide the error of $\vec w_k$ to two parts, error outside the band  and error inside of the band. That is
\begin{align*}
\error_D(\vec w_k) &=  \Pr_{\vec x \sim D} [ \vec x\notin S \text{ and } (\vec w_k\cdot \vec x)(\vec w^*\cdot \vec x)<0] + \Pr_{\vec x \sim D}[ \vec x\in S\text{ and } (\vec  w_k\cdot \vec x)(  \vec w^*\cdot \vec x )<0] .
\end{align*}
For the first part, i.e., error outside of the band, $\Pr_{\vec x \sim D} [ \vec x\notin S \text{ and } (\vec w_k\cdot \vec x)(\vec w^*\cdot \vec x)<0]$ is at most
\begin{align*}
 \Pr_{\vec x \sim D} [ \vec x\notin S \text{ and } (\vec w_k\cdot \vec x)( \vec w_{k-1}\cdot \vec x)<0] + \Pr_{\vec x \sim D} [ \vec x\notin S\text{ and } (\vec w_{k-1}\cdot \vec x)(\vec w^*\cdot \vec x)<0] \leq
\frac{2\alpha_k}{\pi}  e^{-\frac{c^2(d-2)}{2d}},
\end{align*}
where this inequality holds by the application of Lemma~\ref{lem:disagree-outside-band} and the fact that $\theta(\vec w_{k-1}, \vec w_k)\leq \alpha_k$ and $\theta(\vec w_{k-1}, \vec w^*) \leq  \alpha_k$.

For the second part, i.e., error inside the band
\begin{align*}
\Pr_{\vec x \sim D}[ \vec x\in S &\text{ and } (\vec  w_k\cdot \vec x)( \vec w^*\cdot \vec x)<0 ]=
\error_{D_k} (\vec w_k) \Pr_{\vec x \sim D}[\vec x \in S] \\
&\leq \error_{D_k} (\vec w_k) \frac{V_{d-1}}{V_d} 2~b_{k-1} \qquad \text{(By Lemma~\ref{lem:active:prob-in-band})}\\
& \leq  \error_{D_k} (\vec w_k)  ~c~ \alpha_k \sqrt{\frac{2(d+1)}{\pi d}},
%\\ &\leq \left(1.1 \times 2^C \frac{\tau_k}{b_{k-1}} + 3.06 \sqrt 2 \sqrt{1-\beta} \frac{b_{k-1}}{\tau_k} + 3 \times 10^{-6}\right),
\end{align*}
where the last transition holds by the  fact that $\frac{V_{d-1}}{V_d}\leq \sqrt{\frac{d+1}{2\pi}}$~\citep{borgwardt1987simplex}.
Replacing an upper bound on $\error_{D_k}(\vec w_k)$  from Lemma~\ref{lem:active:error-clean-in-band}, to show that $\error_D(\vec w_k) \leq \frac{\alpha_{k+1}}{\pi}$, it suffices to show that the following inequality holds.
\[ \left(
0.757941 \frac{\tau_k}{b_{k-1}}    + 3.303 \sqrt{1-\beta} \frac{b_{k-1}}{\tau_k} + 3.28 \times \gen
 \right)~c~ \alpha_k \sqrt{\frac{2(d+1)}{\pi d}} + \frac{2\alpha_k}{\pi}e^{-\frac{c^2(d-2)}{2d}} \leq \frac{\alpha_{k+1}}{\pi}.
\]
We simplify this inequality as follows.
\[ \left(
0.757941 \frac{\tau_k}{b_{k-1}}    + 3.303 \sqrt{1-\beta} \frac{b_{k-1}}{\tau_k} + 3.28 \times \gen
 \right)~c~  \sqrt{\frac{2\pi (d+1)}{d}} + 2e^{-\frac{c^2(d-2)}{2d}} \leq 1-\lambda.
\]
Replacing in the r.h.s., the values of $c = \cb$, and  $\tau_k = \sqrt{2.50306}   (3.6 \times 10^{-6})^{1/4}    b_{k-1}$, we have
\begin{align*}
%&\left( \frac{\tau_k}{b_{k-1}}    + 3.05 \sqrt{1-\beta} \frac{b_{k-1}}{\tau_k} + 3.067\times 10^{-6}
% \right)~c~  \sqrt{\frac{2\pi (d+1)}{d}} + 2e^{-\frac{c^2(d-2)}{2d}} \\
 &\left( \sqrt{2.50306} (3.6 \times 10^{-6})^{1/4}    + \sqrt{2.50306} \frac{\sqrt{1-\beta}}{(3.6 \times 10^{-6})^{1/4} } + 3.28\times \gen
 \right)~c~  \sqrt{\frac{2\pi (d+1)}{d}} + 2e^{-\frac{c^2(d-2)}{2d}} \\
 &\leq 5.88133  \left( 2 \sqrt{2.50306} (3.6\times 10^{-6})^{1/4} + 3.28 \times \gen
 \right)~ \sqrt{\frac{21}{20}} + 0.167935 \qquad \text{(For $d>20$)} \\
%&\leq   \left(  0.808553 + 1.2156 \times 10^{-5}
% \right)\sqrt{\frac{d+1)}{ d}} + 0.167935 \\
&\leq 0.998573 < 1-\lambda
\end{align*}
Therefore, $\error_D(\vec w_k) \leq 0.038709 (1-\lambda)^k$.

\smallskip
\noindent \textbf{Sample complexity analysis}:
We require $m_k$ labeled samples in the band $S_{\vec w_{k-1}, b_{k-1}}$ at round $k$. By Lemma~\ref{lem:active:prob-in-band}, the probability that a randomly drawn sample from $\tilde D$ falls in
$S_{\vec w_{k-1}, b_{k-1}}$ is at least $O(b_{k-1}\sqrt{d}) = O((1-\lambda)^{k-1})$.
Therefore, we need $O((1-\lambda)^{k-1} m_k)$ unlabeled samples to get $m_k$ examples in the band with probability $1-\frac{\delta}{8(k+k^2)}$.
So, the total unlabeled sample complexity is at most
\[ \sum_{k=1}^s O\left( (1-\lambda)^{k-1} m_k \right) \leq  s \sum_{k=1}^s m_k \in O\left(\frac{1}{\epsilon} \log\left(\frac{d}{\epsilon}\right) \left(d+\log\frac{\log (1/\epsilon)}{\delta}\right) \right).
\]
\end{proofof}

\section{\textsf{Average} Does Not Work} \label{sec:average}
Our algorithm described in the previous section uses convex loss minimization (in our case, hinge loss) in the band as an efficient proxy for minimizing the $0/1$ loss. The \textsf{Average} algorithm introduced by \cite{servedio2001efficient} is another computationally efficient algorithm that has provable noise tolerance guarantees under certain noise models and distributions. For example, it achieves arbitrarily small excess error in the presence of random classification noise and monotonic noise when the distribution is uniform over the unit sphere.
Furthermore, even in the presence of a small amount of malicious noise and less symmetric distributions, \textsf{Average} has been used to obtain a weak learner, which can then be boosted to achieve a non-trivial noise tolerance~\citep{KLS09}. Therefore it is natural to ask, \emph{whether the noise tolerance that \textsf{Average} exhibits could be extended to the case of Massart noise under the uniform distribution?}
We answer this question in the negative. We show that the lack of symmetry in Massart noise presents a significant barrier for the one-shot application of \textsf{Average}, even when the marginal distribution is completely symmetric. Additionally, we also  discuss obstacles in incorporating \textsf{Average} as a weak learner with the margin-based technique.

In a nutshell, \textsf{Average} takes $m$ sample points and their respective labels, $W = \{(\vec x^1, y^1), \dots,\allowbreak (\vec x^m, y^m)\}$, and returns $\frac{1}{m} \sum_{i=1}^m \vec x^i y^i$.
Our main result in this section shows that for a wide range of distributions that are very symmetric in nature, including the Gaussian and the uniform distribution, there is an instance of Massart noise under which \textsf{Average} can not achieve an arbitrarily small excess error.

\begin{theorem}\label{thm:active:oneshot_average}\label{thm:average-lower-bound-intro}
For any continuous distribution $D$ with a p.d.f.\ that is  a function of the distance from the origin only, there is a noisy distribution $\tilde D$ over $X\times \{0,1\}$ that satisfies Massart noise condition in Equation~\ref{eq:active:massart} for some parameter $\beta>0$ and \textsf{Average} returns a classifier with excess error  $\Omega(\frac{\beta(1-\beta)}{1+\beta})$.
\end{theorem}
\begin{proof}
Let $\vec w^* = (1, 0, \dots, 0)$ be the target halfspace.
Let the noise distribution be such that
for all $\vec x$, if $x_1x_2<0$ then we flip the label of $\vec x$ with probability $\frac{1-\beta}{2}$, otherwise we keep the label. Clearly, this satisfies Massart noise with parameter $\beta$.
Let $\vec w$ be expected vector returned by \textsf{Average}. We first show that $\vec w$ is far from $\vec w^*$ in angle. Then, using Equation~\ref{eq:angle-excess} we show that $\vec w$ has large excess error.

First we examine the expected component of $\vec w$ that is parallel to $\vec w^*$, i.e., $\vec w \cdot \vec w^* = w_1$. For ease of exposition, we divide our analysis to two cases, one for regions with no noise (first and third quadrants) and second for regions with noise (second and fourth quadrants). Let $E$ be the event that $x_1x_2>0$. By symmetry,  it is easy to see that $\Pr[E] = 1/2$. Then
\[
\E[\vec w \cdot \vec w^*] = \Pr(E)~ \E[\vec w \cdot \vec w^* | E] + \Pr(\bar{E}) ~ \E[\vec w \cdot \vec w^* | \bar{E}]
\]
For the first term, for $\vec x\in E$ the label has not changed. So, $\E[\vec w \cdot \vec w^* | E]  = \E[|x_1|  ~ | E] = \int_0^1 z f(z)$. For the second term,
the label of each point stays the same with probability $\frac{1+\beta}{2}$ and is flipped with probability $\frac{1-\beta}{2}$. Hence, $\E[\vec w \cdot \vec w^* | E]  = \beta~ \E[|x_1|  ~ | E] = \beta \int_0^1 z f(z)$. Therefore, the expected parallel component of $\vec w$ is
$ \E[\vec w \cdot \vec w^*] = \frac{1+\beta}{2} \int_0^1 z f(z)
$

Next, we examine $w_2$, the orthogonal component of $\vec w$ on the second coordinate. Similar to the previous case for the clean  regions  $\E[w_2 | E]  = \E[|x_2|  ~ | E] = \int_0^1 z f(z)$.
Next, for the second and forth quadrants, which are noisy, we have
\begin{align*}
\E_{(\vec x,y)\sim \tilde D}[x_2y | x_1x_2<0 ]
&=  (\frac{1+\beta}{2}) \int_{-1}^0 z \frac{f(z)}{2} +  (\frac{1-\beta}{2}) \int_{-1}^0 (-z) \frac{f(z)}{2} \qquad \text{(Fourth quadrant)} \\
& \phantom{=} +   (\frac{1+\beta}{2}) \int_0^1 (-z)\frac{f(z)}{2} +  (\frac{1-\beta}{2} ) \int_0^1 z \frac{f(z)}{2} \qquad \text{(Second quadrant)}\\
&=  - (\frac{1+\beta}{2}) \int_0^1 z \frac{f(z)}{2} +  (\frac{1-\beta}{2}) \int_0^1 z \frac{f(z)}{2}  \\
& \phantom{=} -   (\frac{1+\beta}{2}) \int_0^1 z \frac{f(z)}{2} +  (\frac{1-\beta}{2} ) \int_0^1 z \frac{f(z)}{2} \qquad \text{(By symmetry)}\\
&= - \beta \int_0^1 z f(z).
\end{align*}
So,
$
w_2 = \left( \frac{1-\beta}{2} \right) \int_0^1 z f(z).
$
Therefore $\theta(w, w^*) = \arctan(\frac{1-\beta}{1+\beta})\geq \frac{1-\beta}{(1+\beta)} $. By Equation~\ref{eq:angle-excess}, we have
$ \error_{\tilde D}(\vec w) - \error_{\tilde D}(\vec w^*) \geq \beta ~ \frac{\theta(\vec w, \vec w^*)}{\pi}  \geq \beta \frac{1-\beta}{\pi(1+\beta)}.
$
\end{proof}

Our margin-based analysis from Section~\ref{sec:margin} relies on using hinge-loss minimization in the band at every round to efficiently find a halfspace $\vec w_k$ that is a weak learner for $D_{k}$, i.e., $\error_{D_k}(\vec w_k)$ is at most a small constant, as demonstrated in Lemma~\ref{lem:active:error-clean-in-band}.
Motivated by this more lenient goal of finding a weak learner, one might ask whether  \textsf{Average}, as an efficient algorithm for finding low error halfspaces, can be incorporated with the margin-based technique in the same way as hinge loss minimization? We argue that the margin-based technique  is inherently incompatible with \textsf{Average}.

The Margin-based technique maintains two key properties at every step: First,
the angle between $\vec w_k$ and $\vec w_{k-1}$ and the angle between $\vec w_{k-1}$and $w^*$ are small, and as a result $\theta(\vec w^*,\vec  w_k)$ is small. Second, $\vec w_k$ is a weak learner with $\error_{D_{k-1}}(\vec w_k)$ at most a small constant.
In our work,  hinge loss minimization in the band guarantees both of these properties simultaneously by limiting its search to the halfspaces that are close in angle to $\vec w_{k-1}$ and limiting its distribution to $D_{\vec w_{k-1}, b_{k-1}}$.
However, in the case of \textsf{Average} as we concentrate in the band $D_{\vec w_{k-1}, b_{k-1}}$ we bias the distributions towards its orthogonal component with respect to $\vec w_{k-1}$. Hence, an upper bound on $\theta(\vec w^*, \vec w_{k-1})$ only serves to assure that most of the data is orthogonal to $\vec w^*$ as well. Therefore, informally speaking, we lose the signal that otherwise could direct us in the direction of $\vec w^*$. More formally, consider the construction from Theorem~\ref{thm:active:oneshot_average} such that $\vec w_{k-1} = \vec w^* = (1, 0, \dots, 0)$. In distribution $D_{\vec w_{k-1}, b_{k-1}}$, the component of $\vec w_k$ that is parallel to $\vec w_{k-1}$ scales down by the width of the band, $b_{k-1}$. However, as most of the probability stays in a band passing through the origin in any log-concave (including Gaussian and uniform) distribution, the orthogonal component of $\vec w_k$ remains almost unchanged. Therefore, $\theta(\vec w_k, \vec w^*) = \theta(\vec w_k, \vec w_{k-1})
\in \Omega(\frac{1-\beta}{b_{k-1}(1+\beta)}) \geq \left( \frac{(1-\beta) \sqrt d}{(1+\beta) \alpha_{k-1}} \right)$.

\section{Hinge Loss Minimization Does Not Work}
\label{hinge}
Hinge loss minimization is a widely used technique in Machine Learning.
In this section, we show that, perhaps surprisingly,  hinge loss minimization does not lead to arbitrarily small excess error even under very small noise condition, that is it is not consistent.
(Note that in our setting of Massart noise, consistency is the same as achieving arbitrarily small excess error, since the Bayes optimal classifier is a member of the class of halfspaces).

It has been shown earlier that hinge loss minimization can lead to classifiers of large $0/1$-loss \citep{Ben-DavidLSS12}.
However, the lower bounds in that paper  employ distributions with \emph{significant mass} on \emph{discrete points} with \emph{flipped label} (which is not possible under Massart noise) at a very \emph{large distance} from the optimal classifier.
Thus, that result makes strong use of the hinge loss's sensitivity to errors at large distance.
Here, we show that hinge loss minimization is bound to fail under much more benign conditions.
More concretely, we show that for every parameter $\tau$, and arbitrarily small bound on the probability of flipping a label, $\eta = \frac{1-\beta}{2}$,  hinge loss minimization is not consistent even on distributions with a uniform marginal over the unit ball in $\R^2$, with the Bayes optimal classifier being a halfspace and the noise satisfying the Massart noise condition with bound $\eta$. That is, there exists a constant $\epsilon \geq 0$ and a sample size $m(\epsilon)$ such that hinge loss minimization  returns a classifier of excess error at least $\epsilon$ with high probability over sample size of at least $m(\epsilon)$.

Hinge loss minimization
does approximate the optimal hinge loss.
We show that this does not translate into an agnostic learning guarantee for halfspaces with respect to the $0/1$-loss even under very small noise conditions.
Let $\P_\beta$ be the class of distributions $\tilde D$ with uniform marginal over the unit ball $B_1 \subseteq \R^2$, the Bayes classifier being a halfspace $\vec w$, and satisfying the Massart noise condition with parameter $\beta$.
Our lower bound for hinge loss minimization is stated as follows.

\begin{theorem}\label{thm:hinge}
 For every hinge-loss parameter $\tau \geq 0$
 and every Massart noise parameter $0 \leq \beta < 1$, there exists a distribution $\tilde D_{\tau,\beta}\in \P_\beta$
 (that is, a distribution over $B_1 \times \{-1, 1\}$  with uniform marginal over $B_1\subseteq \R^2$ satisfying the $\beta$-Massart condition) such that $\tau$-hinge loss minimization is not consistent on $\tilde D_{\tau,\beta}$ with respect to the class of halfspaces.
That is, there exists an $\epsilon \geq 0$ and a sample size $m(\epsilon)$ such that hinge loss minimization will output a classifier of excess error larger $\epsilon$ (with high probability over samples of size at least $m(\epsilon)$).
\end{theorem}

\paragraph{Proof idea}
To prove the above result, we define a subclass of $\P_{\alpha,\eta} \subseteq \P_\beta$ consisting of well structured distributions.
We then show that for every hinge parameter $\tau$ and every bound on the noise $\eta$, there is a distribution $\tilde D\in \cal P_{\alpha, \eta}$ on which $\tau$-hinge loss minimization is not consistent.

\begin{wrapfigure}[9]{r}{0.2\textwidth}
\label{f:distribution_alpha}
\vspace{-0.8cm}
\centering
\includegraphics[width =.2\textwidth]{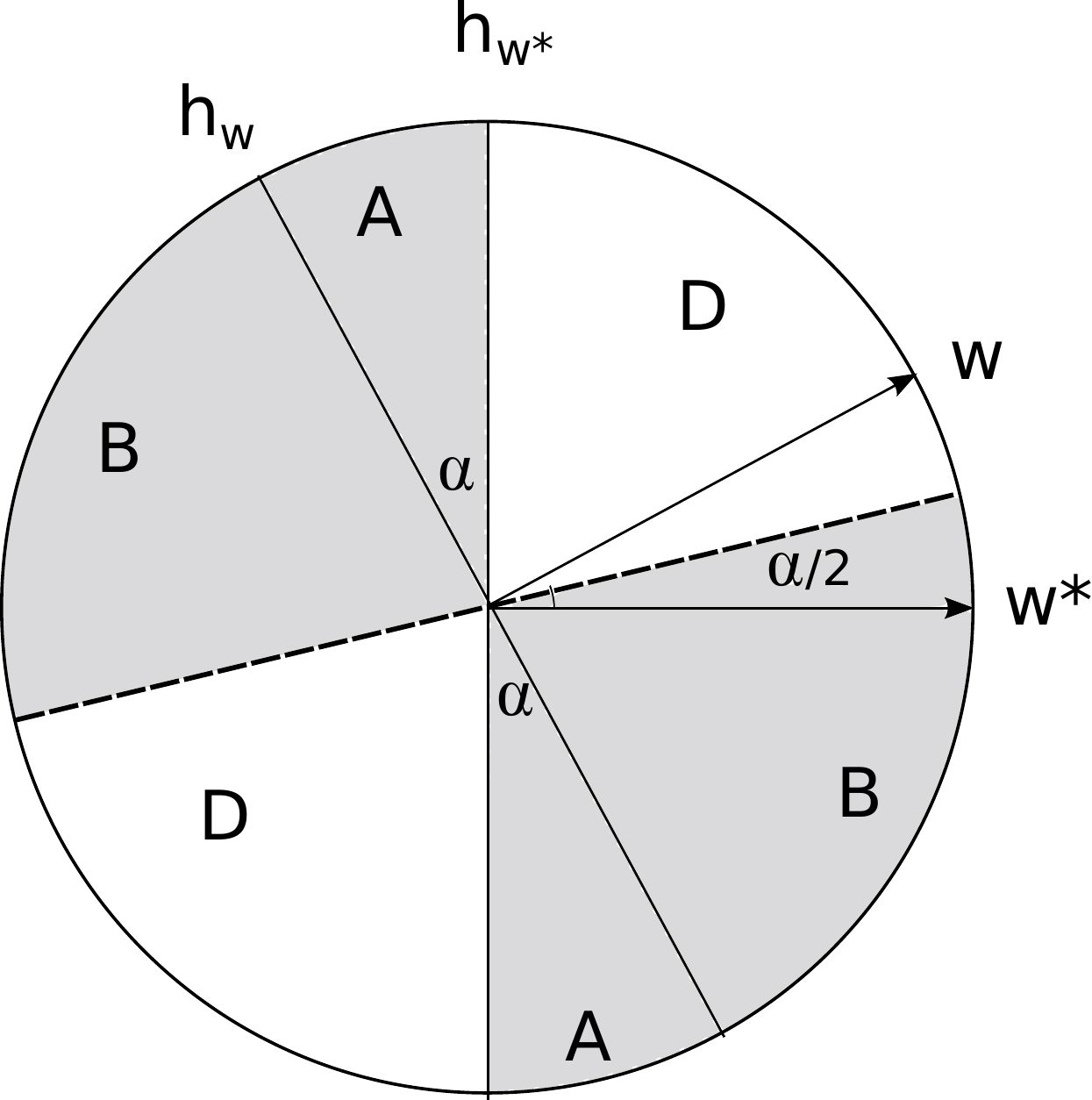}
\caption{$P_{\alpha,\eta}$}
\end{wrapfigure}
In the remainder of this section, we use the notation $h_{\vec w}$ for the classifier associated with a vector $\vec w\in B_1$, that is $h_{\vec w}(x) = \sign(\vec w \cdot x)$, since for our geometric construction it is convenient to differentiate between the two.
We define a family $\P_{\alpha,\eta} \subseteq \P_\beta$ of distributions $\tilde D_{\alpha,\eta}$, indexed by an angle $\alpha$ and a noise parameter $\eta$ as follows.
Let the Bayes optimal classifier be linear $h^* = h_{\vec w^*}$ for a unit vector $\vec w^*$.
Let $h_{\vec w}$ be the classifier that is defined by the unit vector ${\vec w}$ at angle $\alpha$ from ${\vec w}^*$.
We partition the unit ball into areas $A$, $B$ and $D$ as in the Figure \ref{f:distribution_alpha}.
That is $A$ consists of the two wedges of disagreement between $h_{\vec w}$ and $h_{{\vec w}^*}$ and the wedge where the two classifiers agree is divided into $B$ (points that are closer to $h_{\vec w}$ than to $h_{{\vec w}^*}$) and $D$ (points that are closer to $h_{{\vec w}^*}$ than to $h_{\vec w}$).
We now flip the labels of all points in $A$ and $B$ with probability $\eta = \frac{1-\beta}{2}$ and leave the labels deterministic according to $h_{\vec w^*}$ in the area $D$.

More formally,
points at angle between $\alpha/2$ and $\pi/2$ and  points at angle between $\pi + \alpha/2$ and $-\pi/2$ from $\vec w^*$ are labeled per $h_{{\vec w}^*}(\vec x)$ with conditional label probability $1$.
All other points are labeled $-h_{{\vec w}^*}(\vec x)$ with probability $\eta$ and $h_{{\vec w}^*}(\vec x)$ with probability $(1-\eta)$. Clearly, this distribution satisfies Massart noise conditions in Equation~\ref{eq:active:massart} with parameter $\beta$.

The goal of the above construction is to design distributions where vectors along the direction of $\vec w$ have smaller hinge loss of those along the direction of $\vec w^*$.
Observe that the noise in the are $A$ will tend to ``even out'' the difference in hinge loss between $\vec w$ and $\vec w^*$ (since are $A$ is symmetric with respect to these two directions).
The noise in area $B$ however will ``help $\vec w$'': Since all points in area $B$ are closer to the hyperplane defined by $\vec w$ than to the one defined by $\vec w^*$, vector $\vec w^*$ will pay more in hinge loss for the noise in this area.
In the corresponding area $D$ of points that are closer to the hyperplane defined by $\vec w^*$ than to the one defined by $\vec w$ we do not add noise, so the cost for both $w$ and $w^*$ in this area is small.

We show that for every $\alpha$, from a certain noise level $\eta$ on, $w^*$(or any other vector in its direction) is not the expected hinge minimizer on $\tilde D_{\alpha,\eta}$.
We then argue that thereby hinge loss minimization will not approximate $\vec w^*$ arbitrarily close in angle and can therefore not achieve arbitrarily small excess $0/1$-error.
Overall, we show that for every (arbitrarily small) bound on the noise $\eta_0$ and hinge parameter $\tau_0$, we can choose an angle $\alpha$ such that $\tau_0$-hinge loss minimization is not consistent for distribution $\tilde D_{\alpha,\eta_0}$.
The details of the proof can be found in the Appendix, Section \ref{a:hinge}.
\section{Conclusions}
Our work is the first to provide a computationally efficient algorithm under the Massart noise model, a distributional assumption that has been identified in statistical learning to yield fast (statistical) rates of convergence. While both computational and statistical efficiency is crucial in machine learning applications, computational and statistical complexity have been studied under disparate sets of assumptions and models. We view our results on the computational complexity of learning under Massart noise also as a step towards bringing these two lines of research closer together. We hope that this will spur more work identifying situations that lead to both computational and statistical efficiency to ultimately shed light on the underlying connections and dependencies of these two important aspects of automated learning.

\paragraph{Acknowledgments} This work was supported in part by NSF grants CCF-0953192, CCF-1451177, CCF-
1422910, a Sloan Research Fellowshp,  a Microsoft
Research Faculty Fellowship, and a Google Research Award.

\bibliography{one}
\bibliographystyle{plain}

\appendix
\section{Probability Lemmas For The Uniform Distribution}
The following probability lemmas are used throughout this work.
Variation of these lemmas are presented in previous work in terms of their asymptotic behavior \citep{awasthi2014power, balcan2007margin, kalai2008agnostically}. Here, we  focus on finding bounds that are tight even when the constants are concerned.
Indeed, the improved constants in these bounds  are essential to tolerating Massart noise with $\beta > 1- 3.6 \times 10^{-6}$.

Throughout this section, let $D$ be the uniform distribution  over a $d$-dimensional ball. Let $f(\cdot)$ indicate the p.d.f. of $D$. 
For any $d$, let $V_d$ be the volume of a $d$-dimensional unit ball.  Ratios between volumes of the unit ball in different dimensions are commonly used to find the probability mass of different regions under the uniform distribution. Note that for any $d$
\[ \frac{V_{d-2}}{V_d} = \frac{d}{2\pi }.
\]
The following bound due to \cite{borgwardt1987simplex} proves useful in our analysis.
\[ \sqrt{\frac{d}{2\pi}} \leq \frac{V_{d-1}}{V_d} \leq \sqrt{\frac{d+1}{2\pi}}
\]
The next lemma provides an upper and lower bound for the probability mass of a band in uniform distribution. 
\begin{lemma}\label{lem:active:prob-in-band}
Let $\vec u$ be any unit vector in  $\R^d$.
For all $a, b\in[-\frac{C}{\sqrt d}, \frac{C}{\sqrt d} ]$, such that  $C<d/2$, we have
\begin{align*}
|b-a|  2^{-C}   \frac{V_{d-1}}{V_d} \leq \Pr_{\vec x\sim D} &[\vec u\cdot \vec x \in [a,b]]\leq |b-a|\frac{V_{d-1}}{V_d}.
\end{align*}
\end{lemma}
\begin{proof}
We have
\[
\Pr_{\vec x\sim D} [\vec u\cdot \vec x \in [a,b]] = \frac{V_{d-1}}{V_d} \int_a^b (1-z^2 )^{(d-1)/2} ~dz.
\]
For the upper bound, we note that the integrant is at most $1$, so $\Pr_{\vec x\sim D} [\vec u\cdot \vec x \in [a,b]]\leq \frac{V_{d-1}}{V_d} |b-a|$ . For the  lower bound, note that since $a, b\in[-\frac{C}{\sqrt d}, \frac{C}{\sqrt d} ]$, the integrant is at least $(1- \frac{C}{d})^{(d-1)/2}$.
We know that for any $x\in [0, 0.5]$, $1-x > 4^{-x}$. So, assuming that $d>2 C$, 
$(1- \frac{C}{d})^{(d-1)/2} \geq 4^{-\frac{C}{d} (d-1)/2} \geq 2^{-C}$
$\Pr_{\vec x\sim D} [\vec u\cdot \vec x \in [a,b]]  \geq |b-a| 2^{-C}  \frac{V_{d-1}}{V_d}$.
\end{proof}
\begin{lemma}\label{lem:disagree-outside-band}
Let $\vec u$ and $\vec v$ be two unit vectors in $\R^d$ and let $ \alpha = \theta(\vec u, \vec v) $. Then,
\[ \Pr_{x\sim D} [ \sign(\vec u\cdot\vec  x) \neq \sign(\vec w\cdot \vec x) \text{ and } |\vec u\cdot\vec  x|>\frac{c ~ \alpha}{\sqrt d} ] \leq  \frac{\alpha}{\pi} e^{-\frac{c^2(d-2)}{2d}}
\]
\end{lemma}
\begin{proof}
Without the loss of generality, we can assume $\vec u = (1, 0, \dots, 0)$ and $\vec w = (\cos(\alpha), \sin(\alpha), 0 ,\allowbreak\dots, 0)$.
Consider the projection of $D$ on the first $2$ coordinates. Let $E$ be the event we are interested in. 
We first show that for any $\vec x = (x_1, x_2) \in E$, $\Vert \vec x \Vert_2 > c/\sqrt d$. 
Consider $x_1\geq 0$ (the other case is symmetric). If $\vec x\in E$, it must be that $\Vert \vec x \Vert_2 \sin(\alpha) \geq \frac{c \alpha}{\sqrt d}$. So, $\Vert \vec x \Vert_2  =  \frac{c ~\alpha}{\sin(\alpha) \sqrt d} \geq \frac{c}{\sqrt d}$.

Next, we consider a circle  of radius $\frac{c}{\sqrt d} < r <1$ around the center, indicated by $S(r)$. Let $A(r) = S(r) \cap E$ be the arc of such circle that is in $E$.
Then the length of such arc is the arc-length that falls in the disagreement region, i.e., $r\alpha$, minus the arc-length that falls in the band of width $\frac{c\alpha}{\sqrt d}$.
Note, that for every $\vec x\in A(r)$, $\Vert \vec x \Vert_2 = r$, so $f(\vec x) = \frac{V_{d-2}}{V_d} (1- \Vert \vec x \Vert^2)^{(d-2)/2} =  \frac{V_{d-2}}{V_d} (1-r^2)^{(d-2)/2}$.

\allowdisplaybreaks
\begin{align*}
\Pr_{x\sim D}  [ \sign(\vec u\cdot \vec x) \neq & \sign(\vec w\cdot \vec  x) \text{ and } |\vec  u\cdot \vec  x|>\frac{\alpha}{\sqrt d} ] = 2 \int_{\frac {c}{\sqrt d}}^1 (r \alpha - \frac{c\alpha}{\sqrt d}) f(r) ~dr \\
&= 2 \int_1^{\sqrt d / c} (\frac{rc}{\sqrt d}  \alpha - \frac{c\alpha}{\sqrt d}) f(\frac{cr}{\sqrt d}) \frac{c}{\sqrt d} ~dr \quad \text{ (change of variable $z = r \sqrt d/c$ )} \\
&= 2 \frac{V_{d-2}}{V_d}  \frac{c^2 \alpha}{d}    \int_1^{\sqrt d/c} (r-1)  (1 - \frac{c^2r^2}{d})^{(d-2)/2} ~dr \\
&=  \frac{c^2 \alpha}{\pi}    \int_1^{\sqrt d/c} (r-1)  e^{ - \frac{r^2(d-2)}{2d}} ~dr \\
 &\leq  \frac{c^2 \alpha}{\pi}    \int_1^{\sqrt d} \frac{(r-1)}{\frac{(d-2)c^2r}{d}} (-1)(\frac{-(d-2)c^2r}{d})  e^{-\frac{(d-2) c^2r^2}{2d}} ~dr \\
 &\leq  \frac{\alpha}{\pi}    \int_1^{\sqrt d/c} (-1)(\frac{-(d-2)c^2r}{d})  e^{-\frac{(d-2)c^2 r^2}{2d}} ~dr    \\
 &\leq \frac{\alpha}{\pi}  \Big[ -e^{-\frac{(d-2)r^2}{2d}} \Big]_{r =1}^{r = \sqrt d/c} \\
 &\leq \frac{\alpha}{\pi} (e^{-\frac{c^2(d-2)}{2d}} - e^{-(d-2)/2})\\
 & \leq  \frac{\alpha}{\pi} e^{-\frac{c^2(d-2)}{2d}}
\end{align*}
\end{proof}

\section{Proofs of Margin-based Lemmas} \label{app:margin_uniform}

\begin{proofof}{of Lemma~\ref{lem:active:L(w*)}}
Let $L(\vec w^*) = \E_{(\vec x, y)\sim D_k } \ell(\vec w^*,\vec x, y)$, 
$\tau = \tau_k$, and $b = b_{k-1}$.
First note that for our choice of $b\leq \cb \times \ao \frac{1}{\sqrt d}$, using Lemma~\ref{lem:active:prob-in-band} we have that 
\[ \Pr_{x\sim D}[|\vec w_{k-1}\cdot \vec x|< b] \geq 2 ~b \times  2^{-0.285329}.
\]
Note that $L(\vec w^*)$ is maximized when $\vec w^* = \vec w_{k-1}$. Then 
\begin{align*}
L(\vec w^*) \leq  \dfrac{ 2 \int_0^\tau (1 - \frac{a}{\tau}) f(a) ~da}{ \Pr_{x\sim D}[|\vec w_{k-1}\cdot \vec x|< b]   }
 \leq \dfrac{ \int_0^{\tau} (1 - \frac{a}{\tau}) (1- a^2)^{-(d-1)/2} ~da }{  b ~ 2^{-0.285329}}.
\end{align*}
For the numerator:
\begin{align*}
\int_0^{\tau}& (1 - \frac{a}{\tau})  (1- a^2)^{-(d-1)/2} ~da \leq \int_0^\tau (1 - \frac{a}{\tau}) e^{-a^2 (d-1)/2} ~da\\
 &\leq \frac 12 \int_{-\tau}^\tau  e^{-a^2 (d-1)/2} ~da - \frac{1}{\tau} \int_0^\tau a  e^{-a^2 (d-1)/2}   ~da \\
 &\leq \sqrt{\frac{ \pi}{2(d-1)}} ~ \erf \left(\tau \sqrt{\frac{d-1}{2}} \right) - \frac{1}{(d-1) \tau} (1 - e^{-(d-1) \tau^2/2}) \\
 &\leq \sqrt{\frac{ \pi}{2(d-1)}} \sqrt{1- e^{-\tau^2 (d-1)}}  -  \frac{1}{(d-1) \tau} \left(\frac{(d-1) \tau^2}{2} - \frac 12 (\frac{(d-1) \tau^2}{2})^2\right) \quad \text{ (By Taylor expansion)} \\
% &\leq \tau \sqrt{\frac{\pi}{2}} -  \frac{1}{(d-1) \tau} \left(\frac{(d-1) \tau^2}{2} - \frac 12 \left( \frac{(d-1) \tau^2}{2} \right)^2\right) \qquad \text{ (Taylor expansion of $e^{-x}$)} \\
  &\leq \tau \sqrt{\frac{\pi}{2}} -  \frac{\tau}{2} + \frac 18 (d-1) \tau^3\\ 
 &\leq \tau (0.5462  + \frac 18 (d-1) \tau^2) \\
  &\leq 0.5463 \tau \qquad \text{(By $\frac 18 (d-1) \tau^2 < 2 \times 10^{-4}$)}
\end{align*}
Where the last inequality follows from the fact that for our choice of parameters $\tau \leq\allowbreak  \frac{ \sqrt{2.50306}   (3.6 \times 10^{-6})^{1/4}  b}{\sqrt d}<\allowbreak \frac{0.003}{\sqrt d}$, so  $\frac 18 (d-1) \tau^2 <  10^{-5}$.
Therefore, 
\[ L(\vec w^* ) \leq  0.5463 \times  2^{0.285329} \frac{\tau}{b} \leq  \Lw\frac{\tau}{b}.
\]
\end{proofof}

\begin{proofof}{of Lemma~\ref{lem:active:error-clean-in-band}}
Note that  the convex loss minimization procedure returns a vector $\vec v_k$ that is not necessarily normalized. 
To consider all vectors in $B(\vec w_{k-1}, \alpha_k)$, at step $k$, the optimization is done over all vectors  $\vec v$ (of any length) such that  $\Vert \vec w_{k-1}-\vec v \Vert < \alpha_k$.
For all $k$, $\alpha_k< 0.038709\pi$ (or $\ao$), so $\Vert \vec v_k \Vert_2 \geq 1 - \ao $, and as a result $\ell(\vec w_k, W) \leq  1.13844 ~\ell(\vec v_k, W)$.
We have,
\allowdisplaybreaks
\begin{align*}
&\error_{D_k} (\vec w_k) \leq \E_{(\vec x,y)\sim D_k} \ell(\vec w_k, \vec x, y) \\
   &\leq \E_{(\vec x,y)\sim\tilde D_k} \ell(\vec w_k, \vec x, y) + \left( 1.092  \sqrt 2 \sqrt{1-\beta} \frac{b_{k-1}}{\tau_k}  \right)  \qquad \text{(By Lemma~\ref{lem:active:diff-clean-dirty})}  \\
     &\leq \ell(\vec w_k, W) + 1.092  \sqrt 2 \sqrt{1-\beta} \frac{b_{k-1}}{\tau_k}  + \gen \qquad \text{(By Equation~\ref{eq:generalization-dirty})} \\
     &\leq 1.13844 ~\ell(\vec v_k, W) +   1.092  \sqrt 2 \sqrt{1-\beta} \frac{b_{k-1}}{\tau_k}  + \gen \qquad \text{(By $\| \vec v_k \|_2 \geq 1 - \ao $)}\\
     &\leq 1.13844 ~\ell(\vec w^*, W) +  1.092  \sqrt 2 \sqrt{1-\beta} \frac{b_{k-1}}{\tau_k}  + 2.14 \times \gen \quad \text{(By $\vec v_k$ minimizing the hinge-loss)}\\
    &\leq 1.13844 ~\E_{(\vec x,y)\sim\tilde D_k} \ell(\vec w^*, \vec x, y)   +  1.092  \sqrt 2 \sqrt{1-\beta} \frac{b_{k-1}}{\tau_k}  + 3.28 \times \gen \qquad \text{(By Equation~\ref{eq:generalization-dirty})} \\
    &\leq 1.13844 ~ \E_{(\vec x,y)\sim D_k} \ell(\vec w^*, \vec x, y)   +  2.13844 \left( 1.092  \sqrt 2 \sqrt{1-\beta} \frac{b_{k-1}}{\tau_k} \right)  + 3.28\times 10^{-6} \quad  \text{(By Lemma~\ref{lem:active:diff-clean-dirty})} \\
     &\leq  0.757941 \frac{\tau_k}{b_{k-1}}    + 3.303 \sqrt{1-\beta} \frac{b_{k-1}}{\tau_k} + 3.28 \times \gen \quad  \text{(By Lemma~\ref{lem:active:L(w*)})}
\end{align*}
\end{proofof}

\begin{lemma}\label{lem:active:hinge-generalization}
For any constant $c'$, there is $m_k \in O(d(d+ \log(k/d)))$ such that for a randomly drawn set $W$ of $m_k$ labeled samples from $\tilde D_k$,  with probability $1- \frac{\delta}{k + k^2}$, for any $\vec w\in B(\vec w_{k-1}, \alpha_k)$,
\begin{align*}
| \E_{(\vec x,y)\sim \tilde D_k} \left( \ell(\vec w,\vec x,y) - \ell(\vec w, W)  \right) |  &   \leq c',\\
| \E_{(\vec x,y)\sim D_k} \left( \ell(\vec w,\vec x,y) - \ell(\vec w, cleaned(W))   \right)  |  &   \leq c'.
\end{align*}
\end{lemma}
\begin{proof}
By Lemma H.3 of \cite{awasthi2014power}, $\ell(\vec w, \vec x,y) = O(\sqrt d)$ for all $(\vec x, y)\in S_{\vec w_{k-1}, b_{k-1}}$ and $\theta(\vec w, \vec w_{k-1}) \leq r_k$. We get the result by applying  Lemma H.2 of \cite{awasthi2014power}. 
\end{proof}

\section{Initialization}\label{sec:KLS}
We initialize our margin based procedure with the algorithm from~\cite{KLS09}. The guarantees mentioned in \cite{KLS09} hold as long as the noise rate is $\eta \leq c \frac{\epsilon^2}{\log 1/\epsilon}$. \cite{KLS09} do not explicitly compute the constant but it is easy to check that $c \leq \frac 1 {256}$. This can be computed from inequality $17$ in the proof of Lemma~$16$ in \cite{KLS09}. We need the l.h.s. to be at least $\epsilon^2/2$. On the r.h.s., the first term is lower bounded by $\epsilon^2/512$. Hence, we need the second term to be at most $\frac{255}{512}\epsilon^2$. The second term is upper bounded by $4c^2\epsilon^2$. This implies that $c \leq 1/256$.

\section{Hinge Loss Minimization}\label{a:hinge}
In this section, we show that hinge loss minimization is not consistent in our setup, that is, that it does not lead to arbitrarily small excess error.
We let $B_1^d$ denote the unit ball in $R^d$. In this section, we will only work with $d=2$, thus we set $B_1 = B_1^2$.

Recall that the $\tau$-hinge loss of a vector $\vec w \in \R^d$ on an example $(\vec x,y) \in \R^d\times\{-1, 1\}$ is defined as follows:
\[
 \ell_\tau(\vec w ,\vec x, y) = \max \left\{0,~ 1 - \frac{y(\vec w\cdot \vec x)}{\tau} \right\}
\]
For a distribution $\tilde D$ over $\R^d \times \{-1, 1\}$, we let $\hL_\tau^{\tilde D}$ denote the expected hinge loss over $D$, that is
\[
 \hL_\tau^{\tilde D} (\vec w)= \E_{(\vec x,y)\sim {\tilde D}} \ell_\tau(\vec w ,\vec x, y).
\]
If clear from context, we omit the superscript and write $\hL_\tau (\vec w)$ for $\hL_\tau^{\tilde D} (\vec w)$.

Let $\A_\tau$ be the algorithm that minimizes the empirical $\tau$-hinge loss over a sample.
That is, for $W = \{(\vec x_1, y_1), \ldots, (\vec x_m, y_m)\}$, we have
\[
 \A_\tau(W) \in \argmin_{\vec w \in B_1} \frac{1}{|W|} \sum_{(\vec x,y) \in W} \ell_\tau(\vec w ,\vec x, y).
\]

Hinge loss minimization over halfspaces converges to the optimal hinge loss over all halfspace (it is ``hinge loss consistent'').
That is, for all $\epsilon >0$ there is a sample size $m(\epsilon)$ such that for all distributions ${\tilde D}$, we have
\[
 \E_{W \sim {\tilde D}^m} [\hL_\tau^{\tilde D} (\A_\tau(W))] \leq \min_{\vec w \in B_1} \hL_\tau^{\tilde D} (\vec w) + \epsilon.
\]

In this section, we show that this does not translate into an agnostic learning guarantee for halfspaces with respect to the $0/1$-loss.
Moreover, hinge loss minimization is not even consistent with respect to the $0/1$-loss even when restricted to a rather benign classes of distributions $\cal P$.
Let $\P_\beta$ be the class of distributions ${\tilde D}$ with uniform marginal over the unit ball in $\R^2$, the Bayes classifier being a halfspace $\vec w$, and satisfying the Massart noise condition with parameter $\beta$.
We show that there is a distribution ${\tilde D}\in \cal P_\beta$ and an $\epsilon \geq 0$ and a sample size $m_0$ such that hinge loss minimization will output a classifier of excess error larger than $\epsilon$ on expectation over samples of size larger than $m_0$. More precisely, for all $m \geq m_0$:
\[
 \E_{W \sim \tilde D^m} [\hL_\tau^{\tilde D} (\A_\tau(W))] > \min_{\vec w \in B_1} \error_{\tilde D} (\vec w) + \epsilon.
\]
Formally, our lower bound for hinge loss minimization is stated as follows.

\medskip
\noindent\textbf{Theorem~\ref{thm:hinge}~(Restated).}
\emph{ For every hinge-loss parameter $\tau \geq 0$
 and every Massart noise parameter $0\leq \beta < 1$, there exists a distribution $\tilde D_{\tau,\beta}\in \P_\beta$
 (that is, a distribution over $B_1 \times \{-1, 1\}$  with uniform marginal over $B_1\subseteq \R^2$ satisfying the $\beta$-Massart condition) such that $\tau$-hinge loss minimization is not consistent on $P_{\tau,\beta}$ with respect to the class of halfspaces.
  That is, there exists an $\epsilon \geq 0$ and a sample size $m(\epsilon)$ such that hinge loss minimization will output a classifier of excess error larger than $\epsilon$ (with high probability over samples of size at least $m(\epsilon)$).
}
\medskip

In the section, we use the notation $h_w$ for the classifier associated with a vector $\vec w\in B_1$, that is $h_w(x) = \sign(\vec w \cdot x)$,
since for our geometric construction it is convenient to differentiate between the two.
The rest of this section is devoted to proving the above theorem.

\subsection*{A class of distributions}
\begin{wrapfigure}[10]{r}{0.2\textwidth}
\label{af:distribution_alpha}
% \begin{center}
%\vspace{-.8cm}
\centering
\includegraphics[width =.2\textwidth]{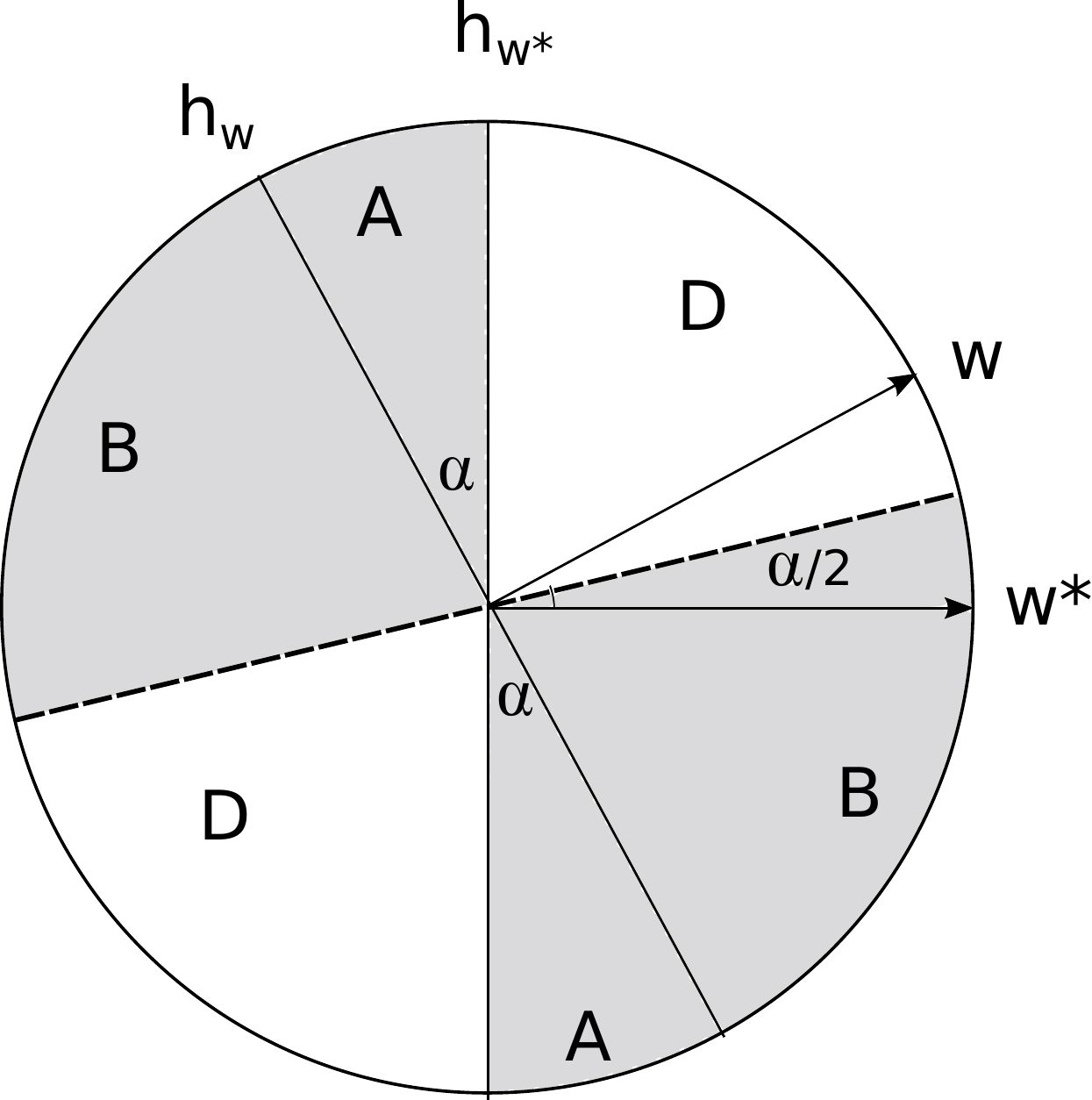}
% \end{center}
\caption{${\tilde D}_{\alpha,\eta}$}
\end{wrapfigure}
 Let $\eta = \frac{1-\beta}{2}$.
 We define a family $\P_{\alpha,\eta} \subseteq \P_\beta$ of distributions ${\tilde D}_{\alpha,\eta}$, indexed by an angle $\alpha$ and a noise parameter $\eta$ as follows.
 We let the marginal be uniform over the unit ball $B_1\subseteq \R^2$ and let the Bayes optimal classifier be linear $h^* = h_{\vec w^*}$ for a unit vector $\vec w^*$.
 Let $h_{\vec w}$ be the classifier that is defined by the unit vector ${\vec w}$ at angle $\alpha$ from ${\vec w}^*$.
 We partition the unit ball into areas $A$, $B$ and $D$ as in the Figure 2.
 That is $A$ consists of the two wedges of disagreement between $h_{\vec w}$ and $h_{{\vec w}^*}$ and the wedge where the two classifiers agree is divided in $B$ (points that are closer to $h_{\vec w}$ than to $h_{{\vec w}^*}$) and $D$ (points that are closer to $h_{{\vec w}^*}$ than to $h_{\vec w}$).
 We now ``add noise $\eta$'' at all points in areas $A$ and $B$ and leave the labels deterministic according to $h_{\vec w^*}$ in the area $D$.

More formally,
points at angle between $\alpha/2$ and $\pi/2$ and  points at angle between $\pi + \alpha/2$ and $-\pi/2$ from $\vec w^*$ are labeled with $h_{\vec w^*}(x)$ with (conditional) probability $1$.
All other points are labeled $-h_{\vec w^*}(x)$ with probability $\eta$ and $h_{\vec w^*}(x)$ with probability $(1-\eta)$.

\subsection*{Useful lemmas}
The following lemma relates the $\tau$-hinge loss of unit length vectors to the hinge loss of arbitrary vectors in the unit ball.
 It will allow us to focus our attention to comparing the $\tau$-hinge loss of unit vectors for $\tau > \tau_0$, instead of having to argue about the $\tau_0$ hinge loss of vectors of arbitrary norms in $B_1$.

 \begin{lemma}\label{l:hinge_scaling}
  Let $\tau >0$ and $0 < \lambda \leq 1$.
  Let $\vec w$ and $\vec w^*$ be two vectors of unit length.
  Then $\hL_{\tau}(\lambda \vec w) < \hL_{\tau}(\lambda \vec w^*)$ if and only if $\hL_{\tau/\lambda}(\vec w) < \hL_{\tau/\lambda}(\vec w^*)$.
 \end{lemma}
 \begin{proof}
 By the definition of the hinge loss, we have
 \[
   \ell_\tau(\lambda \vec w ,\vec x, y) = \max \left(0,~ 1 - \frac{y(\lambda \vec w\cdot \vec x)}{\tau} \right) = \max \left(0,~ 1 - \frac{y( \vec w\cdot \vec x)}{\tau/ \lambda} \right) = \ell_{\tau/\lambda}(\vec w ,\vec x, y).
 \]
 \end{proof}

\begin{lemma}\label{l:angle_away}
Let $\tau >0$, for any ${\tilde D}\in \P_{\alpha,\eta}$ let $w_\tau$ denote the halfspace that minimizes the $\tau$-hinge loss with respect to ${\tilde D}$. If $\theta(w^*, w_\tau)>0$, then hinge loss minimization is not consistent for the $0/1$-loss.
\end{lemma}
\begin{proof}

 First we show that the hinge loss minimizer is never the vector $\vec 0$.   Note that $\hL_\tau^{\tilde D}(\vec 0) =1$ (for all $\tau>0$).
Consider the case $\tau \geq 1$, we show that $\vec w^*$ has $\tau$-hinge loss strictly smaller than $1$. Integrating the hinge loss over the unit ball using polar coordinates, we get
  \begin{align*}
   \hL_\tau^{\tilde D}(\vec w^*) & < \frac{2}{\pi} \left((1-\eta) \int_0^1 \int_0^\pi (1 - \frac{z}{\tau} \sin(\varphi)) ~z ~d\varphi ~dz +  \eta \int_0^1 \int_0^\pi (1 + \frac{z}{\tau} \sin(\varphi)) ~z ~d\varphi ~dz\right)   \\
   & = \frac{2}{\pi} \left((1-\eta) \int_0^1 \int_0^\pi z - \frac{z^2}{\tau} \sin(\varphi) ~d\varphi ~dz +  \eta \int_0^1 \int_0^\pi z + \frac{z^2}{\tau} \sin(\varphi) ~d\varphi ~dz\right)   \\
   & = 1 + \frac{2}{\pi} \left((1-2\eta) \int_0^1 \int_0^\pi - \frac{z^2}{\tau} \sin(\varphi) ~d\varphi ~dz \right)   \\
   & = 1 - \frac{2}{\pi} \left((1-2\eta) \int_0^1 \int_0^\pi  \frac{z^2}{\tau} \sin(\varphi) ~d\varphi ~dz \right)  < 1.
  \end{align*}
For the case of $\tau <1$,  we have
  \[
   \hL_\tau (\tau \vec w^*) = \hL_1(\vec w^*) <1.
  \]
Thus, $(0,0)$ is not the hinge-minimizer. Then, by the assumption of the lemma $w_\tau$ has some positive angle $\gamma$ to the $\vec w^*$.
Furthermore, for all $0\leq \lambda\leq 1$,  $\hL_\tau^{\tilde D}(\vec w_\tau) < \hL_\tau^{\tilde D}(\lambda \vec w^*)$.
Since $\vec w \mapsto \hL_\tau^{\tilde D}(\vec w)$ is a continuous function we can choose an $\epsilon >0$ such that
  \[
   \hL_\tau^{\tilde D}(\vec w_\tau) + \epsilon/2 < \hL_\tau^{\tilde D}(\lambda \vec w^*) - \epsilon/2.
  \]
  for all $0\leq \lambda\leq 1$ (note that the set $\{\lambda\vec w^* \mid 0\leq \lambda \leq 1 \}$ is compact).
  Now, we can choose an angle $\mu < \gamma$ such that for all vectors $\vec v$ at angle at most $\mu$ from $\vec w^*$, we have
  \[
   \hL_\tau^{\tilde D}(\vec v)  \geq  \min_{0\leq \lambda\leq 1}\hL_\tau^{\tilde D}(\lambda \vec w^*) - \epsilon/2
  \]
  Since hinge loss minimization will eventually (in expectation over large enough samples) output classifiers of hinge loss strictly smaller than $\hL_\tau^{\tilde D}(\vec w_\tau) + \epsilon/2$, it will then not output classifiers of angle smaller than $\mu$ to $\vec w^*$. By Equation~\ref{eq:angle-excess}, for all $w$, $\error_{\tilde D}(\vec w) - \error_{\tilde D}(\vec w^*)> \beta \frac{\theta(w, w^*)}{\pi}$, therefore, the excess error of a the classfier returned by hinge loss minimization is lower bounded by a constant $\beta \frac{\mu}{\pi}$.  Thus, hinge loss minimization is not consistent with respect to the $0/1$-loss.
\end{proof}

\subsection*{Proof of Theorem \ref{thm:hinge}}

We will show that, for every bound on the noise $\eta_0$ and for every every $\tau_0\geq 0$
 there is an $\alpha_0 >0$, such that the unit length vector $\vec w$ has strictly lower $\tau$-hinge loss than the unit length vector $\vec w^*$ for all $\tau \geq \tau_0$.
 By Lemma \ref{l:hinge_scaling}, this implies that for every bound on the noise $\eta_0$ and for every $\tau_0$ there is an $\alpha_0 >0$ such that for all $0 < \lambda \leq 1$ we have $\hL_{\tau_0}(\lambda \vec w) < \hL_{\tau_0}(\lambda \vec w^*)$.
This implies that the hinge minimizer is not a multiple of  $\vec w^*$ and so is at a positive angle to $\vec w^*$.
Now Lemma \ref{l:angle_away} tells us that hinge loss minimization is not consistent for the $0/1$-loss.

 \begin{wrapfigure}[10]{r}{0.2\textwidth}
% \begin{center}
\vspace{-.8cm}
\centering
\includegraphics[width =.2\textwidth]{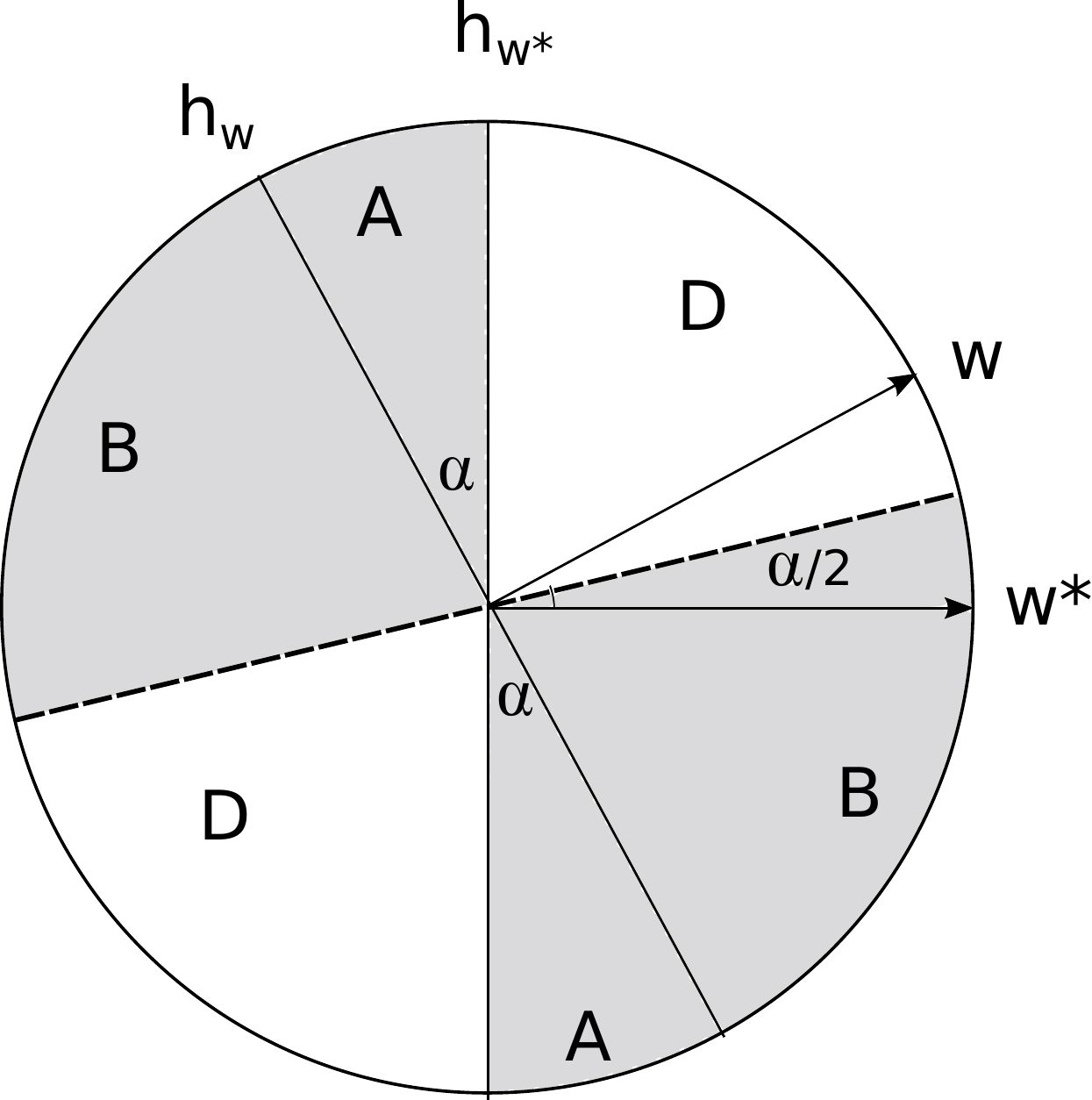}
% \end{center}
\caption{${\tilde D}_{\alpha,\eta}$}
\end{wrapfigure}
 In the sequel, we will now focus on the unit length vectors $\vec w$ and $\vec w^*$ and show how to choose $\alpha_0$ as a function  of $\tau_0$ and $\eta_0$.
We let $\cA$ denote the hinge loss of $h_{\vec w^*}$ on one wedge (one half of) area $A$ when the labels are correct and $\dA$ that hinge loss on that same area when the labels are not correct.
 Analogously, we define $\cB, \dB, \cD$ and $\dD$.
 For example, for $\tau \geq 1$, we have (integrating the hinge loss over the unit ball using polar coordinates)

\begin{align*}
&\cA = \frac{1}{\pi}\int_0^1 \int_0^\alpha (1 -  \frac{z}{\tau} \sin(\varphi))  z~d\varphi ~dz,\\
&\dA = \frac{1}{\pi}\int_0^1 \int_0^\alpha (1 +  \frac{z}{\tau} \sin(\varphi))  z~d\varphi ~dz,\\
& \cB = \frac{1}{\pi} \int_0^1 \int_\alpha^{\frac{\pi + \alpha}{2}} (1 -  \frac{z}{\tau} \sin(\varphi))  z~d\varphi ~dz,\\
&  \dB = \frac{1}{\pi} \int_0^1 \int_\alpha^{\frac{\pi + \alpha}{2}} (1 +  \frac{z}{\tau} \sin(\varphi))  z~d\varphi ~dz,\\
&  \cD = \frac{1}{\pi} \int_0^1 \int_0^{\frac{\pi - \alpha}{2}} (1 -  \frac{z}{\tau} \sin(\varphi))  z~d\varphi ~dz,\\
\text{and}\qquad &\dD = \frac{1}{\pi} \int_0^1 \int_0^{\frac{\pi - \alpha}{2}} (1 +  \frac{z}{\tau} \sin(\varphi))  z~d\varphi ~dz.
 \end{align*}
Now we can express the hinge loss of both $h_{\vec w^*}$ and $h_{\vec w}$ in terms of these quantities.
 For $h_{\vec w^*}$ we have
 \[
  \hL_\tau(h_{\vec w^*}) = 2 \cdot \left(\eta (\dA + \dB) + (1-\eta)(\cA + \cB) + \cD \right).
 \]
 For $h_{\vec w}$, note that area $B$ relates to $h_{\vec w}$ as area $D$ relates to $h_{\vec w^*}$ (and vice versa). Thus, the roles of $B$ and $D$ are exchanged for $h_{\vec w}$. That is, for example, for the noisy version of area $B$ the classifier $h_{\vec w}$ pays $\dD$. We have
 \[
  \hL_\tau(h_{\vec w}) = 2 \cdot \left(\eta (\cA + \dD) + (1-\eta)(\dA + \cD) + \cB \right).
 \]
 This yields
\[
 \hL_\tau(h_{\vec w}) -  \hL_\tau(h_{\vec w^*}) = 2 \cdot \left( (1- 2\eta) (\dA - \cA) - \eta ((\dB - \cB) - (\dD - \cD)) \right).
 \]

\begin{wrapfigure}[7]{r}{0.2\textwidth}
\label{af:areaC}
% \begin{center}
\vspace{-.8cm}
\centering
\includegraphics[width =.2\textwidth]{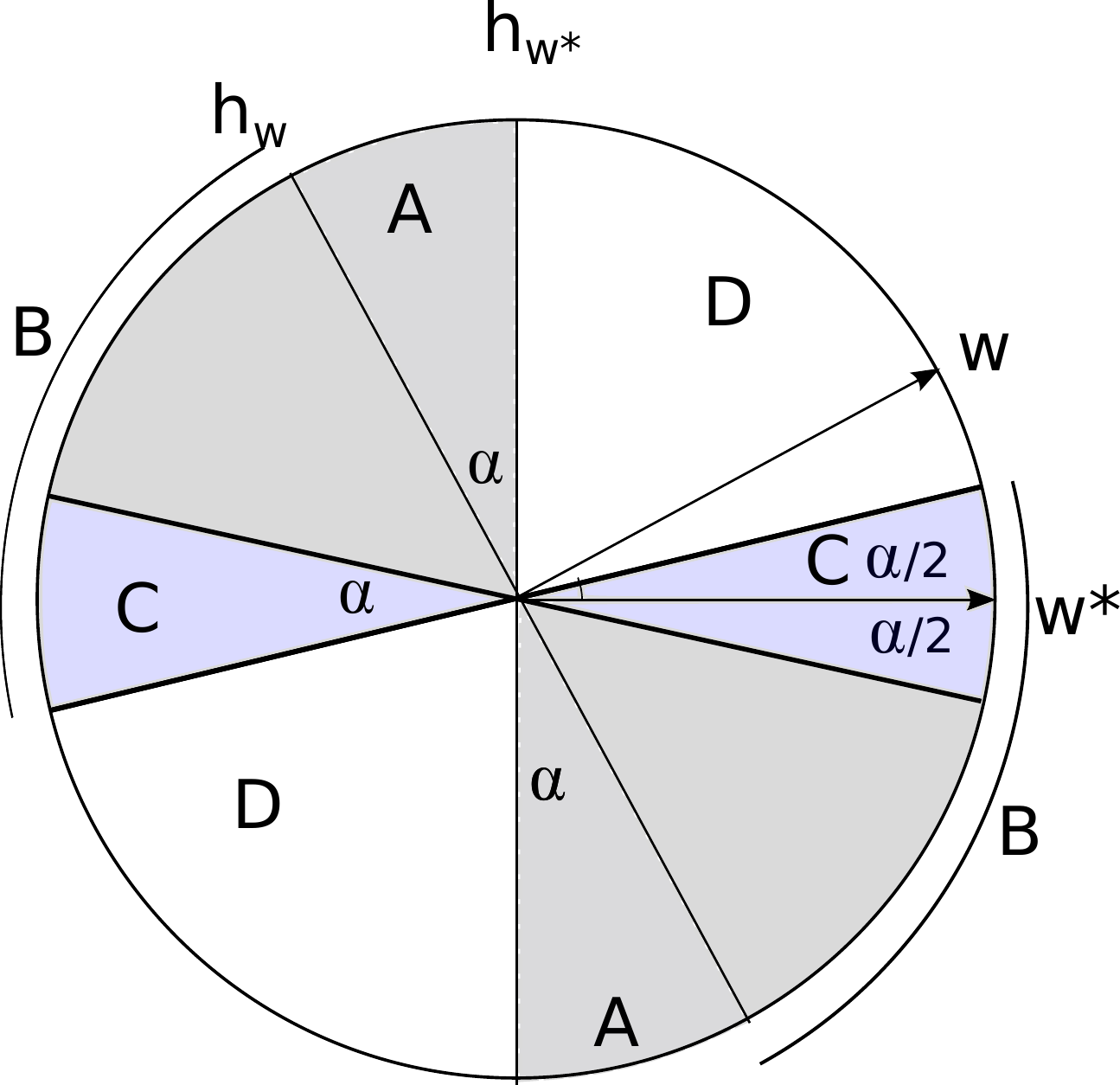}
% \end{center}
\caption{Area $C$}
\end{wrapfigure}
We now define area $C$ as the points at angle between $\pi - \alpha/2$ and $\pi +\alpha/2$ from $\vec w^*$ (See Figure 3).
% (See Figure \ref{af:areaC}).
We let $\cC$ and $\dC$ be defined analogously to the above.

Note that $\dA + \dB -\dD = \dC$ and $\cA + \cB -\cD = \cC$.
Thus we get
\begin{align*}
 &  \hL_\tau(h_{\vec w}) -  \hL_\tau(h_{\vec w^*}) \\
  = & 2 \cdot \left( (1- 2\eta) (\dA - \cA) - \eta ((\dB - \cB) - (\dD - \cD)) \right)\\
  = & 2 \cdot \left( (1- \eta) (\dA - \cA) - \eta ((\dB - \cB) + (\dA - \cA) - (\dD - \cD)) \right)\\
  = & 2 \cdot \left( (1- \eta) (\dA - \cA) - \eta ((\dC - \cC)) \right).
\end{align*}

If $\eta > \eta(\alpha, \tau) := \frac{(\dA - \cA)}{(\dA - \cA) + (\dC - \cC)}$, then we get $\hL_\tau(h_{\vec w}) -  \hL_\tau(h_{\vec w^*}) < 0$ and thus $h_w$ having smaller hinge loss than $h_{\vec w^*}$.
Thus, $\eta(\alpha, \tau)$ signifies the amount of noise from which onward, $\vec w$ will have smaller hinge loss than ${\vec w^*}$

Given $\tau_0 \geq 0$,
choose $\alpha$  small enough (we can always choose the angle $\alpha$ sufficiently small for this) so that the area $A$ is included in the $\tau_0$-band around $\vec w^*$. We have for all $\tau \geq \tau_0$:
\begin{align*}
 (\dA - \cA) & = \frac{2}{\pi} \int_0^1 \int_0^\alpha \frac{z^2}{\tau} \sin(\varphi) ~d\varphi ~dz \\
 & = \frac{2}{3\pi} \int_0^\alpha \frac{1}{\tau} \sin(\varphi) ~d\varphi  \\
 & =  \frac{2}{3\pi\tau} \left[ -\cos(\varphi) \right]_0^\alpha  \\
 &=  \frac{2}{3\pi\tau}(1 - \cos(\alpha)).
\end{align*}

For the area $C$ we  now consider the case of $\tau \geq 1 $ and $\tau <1$ separately.
For $\tau \geq 1$ we get
\begin{align*}
 (\dC - \cC) & = \frac{4}{\pi} \int_0^1 \int_{\frac{\pi-\alpha}{2}}^{\frac{\pi}{2}} \frac{z^2}{\tau} \sin(\varphi) ~d\varphi ~dz \\
 & = \frac{4}{3\pi} \int_{\frac{\pi-\alpha}{2}}^{\frac{\pi}{2}}  \frac{1}{\tau} \sin(\varphi) ~d\varphi  \\
%  & =  \frac{1}{\tau} \left[ -\cos(\varphi) \right]_0^\alpha  \\
 &=  \frac{4}{3\pi\tau} \cos\left(\frac{\pi -\alpha}{2}\right)\\
 &=  \frac{4}{3\pi\tau} \sin\left(\frac{\alpha}{2}\right).
\end{align*}

Thus, for $\tau \geq 1$ we get
\[
 \eta(\alpha, \tau) = \frac{(\dA - \cA)}{(\dA - \cA) + (\dC + \cC)} = \frac{1 - \cos(\alpha)}{1 - \cos(\alpha) + 2 \sin(\alpha/2)}.
\]

We call this quantity $\eta_1(\alpha)$ since, given that $\tau \geq 1$, it does not depend on $\tau$:
\[
 \eta_1(\alpha) = \frac{(\dA - \cA)}{(\dA - \cA) + (\dC + \cC)} = \frac{1 - \cos(\alpha)}{1 - \cos(\alpha) + 2 \sin(\alpha/2)}.
\]
Observe that $\lim_{\alpha \to 0} \eta_1(\alpha) = 0$.
This will yield the first condition on the angle $\alpha$:
Given some bound on the allowed noise $\eta_0$, we can choose an $\alpha$ small enough so that $\eta_1(\alpha) \leq \eta_0/2$.
Then, for the distribution ${\tilde D}_{\alpha, \eta_0}$ we have $\hL_\tau(\vec w) < \hL_\tau(\vec w^*)$ for all
$\tau \geq 1$.

We now consider the case $\tau < 1$.
For this case we lower bound $(\dC - \cC)$ as follows. We have
\begin{align*}
 \dC  & = \frac{2}{\pi} \int_0^1 \int_{\frac{\pi-\alpha}{2}}^{\frac{\pi}{2}} z + \frac{z^2}{\tau} \sin(\varphi) ~d\varphi ~dz \\
& =  \frac{\alpha}{2\pi} + \frac{2}{\pi}  \int_0^1 \int_{\frac{\pi-\alpha}{2}}^{\frac{\pi}{2}} \frac{z^2}{\tau} \sin(\varphi) ~d\varphi ~dz \\
 &=  \frac{\alpha}{2\pi} + \frac{2}{3\tau\pi} \sin\left(\frac{\alpha}{2}\right).
\end{align*}

 \begin{wrapfigure}[10]{r}{0.2\textwidth}
\label{af:areaT}
% \begin{center}
\vspace{-.8cm}
\centering
\includegraphics[width =.2\textwidth]{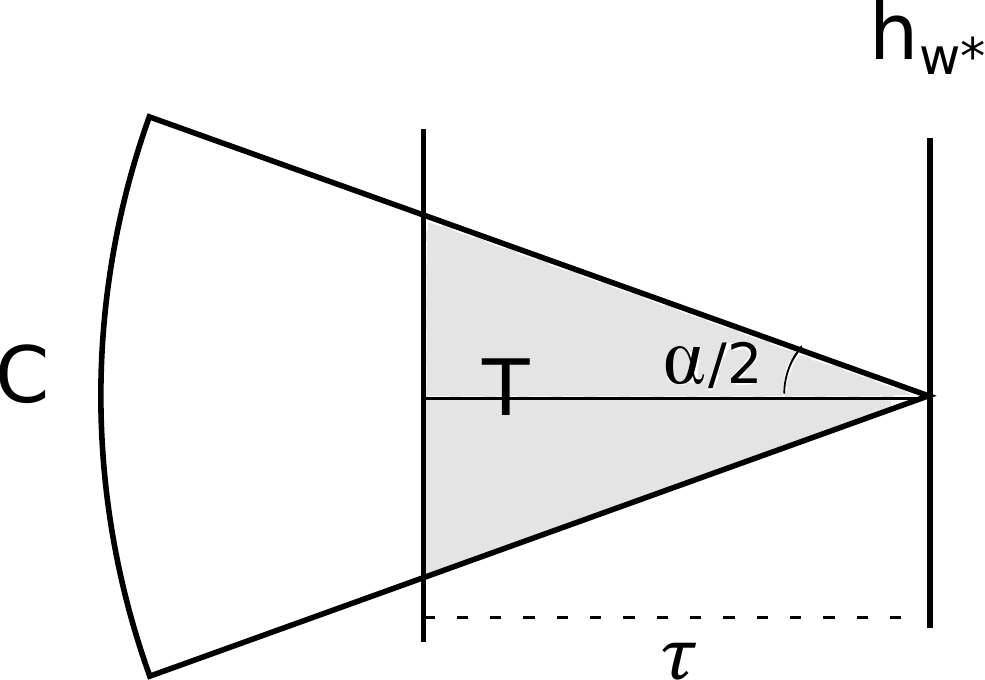}
% \end{center}
\caption{Area $T$}
\end{wrapfigure}
We now provide an upper bound on $\cC$ by integrating over a the triangular shape $T$ (see Figure 4).
% (see Figure \ref{af:areaT}).
Note that this bound on $\cC$ is actually exact if $\tau \leq \cos(\alpha/2)$ and only a strict upper bound for $ \cos(\alpha/2) <\tau < 1$. We have
\begin{align*}
 \cC \leq \mathrm(cT) & = \frac{2}{\pi} \cdot \int_0^\tau  (1 -\frac{z}{\tau}) (z\tan(\alpha/2)) ~dz \\
   & = \frac{2}{\pi}\cdot \int_0^\tau  z\tan(\alpha/2) -\frac{z^2}{\tau}\tan(\alpha/2) ~dz \\
   & = \frac{\tau^2}{3\pi} \tan\left(\frac{\alpha}{2}\right).
 % & =  \alpha + 2\cdot \int_0^1 \int_{\frac{\pi-\alpha}{2}}^{\frac{\pi}{2}} \frac{z}{\tau} \sin(\varphi) ~d\varphi ~dz \\
%  &=  \alpha + \frac{1}{\tau} \sin\left(\frac{\alpha}{2}\right).
\end{align*}
Thus we get
\[
 (\dC - \cC) \geq (\dC - \mathrm(cT)) = \frac{1}{\pi} \left(\frac{\alpha}{2} + \frac{2}{3\tau}\sin\left(\frac{\alpha}{2}\right) - \frac{\tau^2}{3}\tan\left(\frac{\alpha}{2}\right)\right).
\]

This yields, for the case $\tau \leq 1$
\[
 \eta(\alpha, \tau) = \frac{\frac{2}{3}(1-\cos(\alpha))}{\frac{2}{3}(1-\cos(\alpha)) + \frac{2}{3}\sin(\alpha)  + \frac{\alpha\tau}{2} - \frac{\tau^3}{3}  \tan(\frac{\alpha}{2}) }
%                    = \frac{(1-\cos(\alpha))}{(1-\cos(\alpha)) + \sin(\alpha)  + \tau \alpha - \frac{\tau^3}{3}  \tan(\frac{\alpha}{2}) }
 \]
We call this quantity $\eta_2(\alpha, \tau)$ to differentiate it from $\eta_1(\alpha)$.
Again, it is easy to show that we have $\lim_{\alpha \to 0} \eta_2(\alpha, \tau) =0$ for every $\tau$.
Thus, for a fixed $\tau_0$, we can choose an angle $\alpha$ small enough so that $\hL_{\tau_0}(\vec w) \leq \hL_{\tau_0}(\vec w^*)$.

To argue that we will then also have $\hL_{\tau}(\vec w) \leq \hL_{\tau}(\vec w^*)$ for all $\tau \geq \tau_0$, we show that, for a fixed angle $\alpha$, the function $\eta(\alpha, \tau)$ gets smaller as $\tau$ grows.
For this, it suffices to show that $g(\tau) = \tau \frac{\alpha}{2} - \frac{\tau^3}{3}  \tan(\frac{\alpha}{2})$ is monotonically increasing with $\tau$ for $\tau \leq 1$.
We have
\[
 g'(\tau) = \frac{\alpha}{2} - \frac{\tau^2}{2}\tan\left(\frac{\alpha}{2} \right).
\]
Since we have $\tau^2 \leq 1$ and $\frac{2\alpha}{\tan\left(\frac{\alpha}{2} \right)} \geq 1$ for  $0 \leq\alpha\leq \pi/3$, we get that (for sufficiently small $\alpha$) $g'(\tau) \geq 0$ and thus $g(\tau)$ is monotonically increasing for $0\leq \tau \leq 1$ as desired.

Summarizing, for a given $\tau_0$ and $\eta_0$, we can always choose $\alpha_0$ sufficiently small so that both $\eta_1(\alpha_0) < \frac{\eta_0}{2}$ and $\eta_2(\alpha_0, \tau) < \frac{\eta_0}{2}$ for all $\tau \geq \tau_0$
and thus $\hL_\tau^{{\tilde D}_{\alpha_0,\eta_0}}(\vec w) < \hL_\tau^{{\tilde D}_{\alpha_0,\eta_0}}(\vec w^*)$ for all $\tau \geq \tau_0$.
This completes the proof.

\end{document}